\documentclass[runningheads]{llncs}

\usepackage{eccv}

\usepackage{eccvabbrv}

\usepackage{graphicx}
\usepackage{booktabs}

\usepackage{hyperref}

\usepackage{orcidlink}

\usepackage{graphicx}
\usepackage{amssymb}
\usepackage{amsmath}
\usepackage{dsfont}
\usepackage{multirow}
\usepackage{float}
\usepackage{lipsum}
\usepackage{nicefrac}
\usepackage{circledsteps}
\usepackage{enumitem}
\usepackage{colortbl}
\usepackage{graphicx}
\usepackage{booktabs}
\usepackage{siunitx}
\usepackage{breqn}
\usepackage{lipsum}
\usepackage{adjustbox}
\usepackage{capt-of}
\usepackage{duckuments}
\usepackage{tikzducks}
\usepackage{tablefootnote}
\usepackage{pifont} %xmark and cmark
\usepackage{makecell}
\usepackage{tikz}
\usepackage{algorithm}
\usepackage{algorithmic}
\usepackage{tcolorbox}
\usepackage{dsfont}

\definecolor{lightgray}{gray}{0.95}
\definecolor{midgray}{gray}{0.55}
\definecolor{steelblue}{HTML}{4D82B7}
\definecolor{davysgrey}{rgb}{0.33, 0.33, 0.33}
\definecolor{LightCyan}{rgb}{0.88,1,1}
\definecolor{ao(english)}{rgb}{0.0, 0.5, 0.0}

\newcommand{\g}[1]{\cellcolor{gray!10}{#1}}

\usepackage[first=0, last=9]{lcg}

\newcommand{\Star}[1]{#1\ensuremath{^*}\kern-\scriptspace}

\usepackage{tikz}

\newcommand{\cmark}{\textcolor{teal}{\checkmark}}
\newcommand{\xmark}{\textcolor{red}{$\times$}}

\tcbset{
  finding/.style={
    enhanced, breakable,
    colback=mintbg, colframe=mintedge,
    boxrule=0.6pt, arc=6pt, 
    drop fuzzy shadow=shadow, 
    left=8pt,right=8pt,top=6pt,bottom=6pt,
    before skip=6pt, after skip=6pt
  }
}

\makeatletter
\DeclareRobustCommand\onedot{\futurelet\@let@token\@onedot}
\def\@onedot{\ifx\@let@token.\else.\null\fi\xspace}

\def\eg{\emph{e.g}\onedot} 
\def\ie{\emph{i.e}\onedot}

\definecolor{algc1}{HTML}{f7d779}
\definecolor{algc2}{HTML}{9fc5fc}

\renewcommand{\algorithmiccomment}[1]{\bgroup\hfill $\triangleright$ ~#1\egroup}

\usepackage{array}
\newcommand{\PreserveBackslash}[1]{\let\temp=\\#1\let\\=\temp}
\newcolumntype{C}[1]{>{\PreserveBackslash\centering}p{#1}}
\newcolumntype{R}[1]{>{\PreserveBackslash\raggedleft}p{#1}}
\newcolumntype{L}[1]{>{\PreserveBackslash\raggedright}p{#1}}

\newcommand{\method}{\textsc{Sum}\xspace}
\newcommand{\methodlong}{\textsc{Surgery \& Merge}\xspace}

\begin{document}

\title{\method{}: Unified Geometric Surgery on Spatio-Temporal Adaptation Vectors for Federated Class Incremental Learning} 

\titlerunning{\methodlong{}}

\author{
Jaeik Kim\inst{1} \and
Jaeyoung Do\inst{1,2}\thanks{Corresponding author.}
}

\authorrunning{J.~Kim and J.~Do}


\institute{
AIDAS Lab, $^1$IPAI \& $^2$ECE,
Seoul National University, Seoul, Republic of Korea\\
\email{\{jake630,jaeyoung.do\}@snu.ac.kr}
}

\maketitle

% \begin{abstract}
% Federated Learning (FL) often suffers from spatial interference—client drift from heterogeneous data—while Continual Learning (CL) faces temporal interference, or catastrophic forgetting, as tasks arrive sequentially.
% Federated Continual Learning (FCL) addresses both challenges by learning a global model across distributed clients that continually adapts to evolving tasks, a setting naturally encountered in real-world systems. However, most prior methods treat the spatial and temporal aspects separately, lacking a unified optimization view and often requiring extra data or retraining.
% We reinterpret FCL as a unified spatio-temporal Multi-Task Learning (MTL) problem, where client and task updates are represented as gradient-based adaptation vectors under a shared geometric framework.
% Building on this formulation, we propose \methodlong{} (\method{}), which applies a single principle to both dimensions: symmetric projection for concurrent spatial interference and asymmetric online projection for sequential temporal interference, achieving a provably tighter loss bound than standard aggregation.
% \method{} yields up to 22\% improvement over state-of-the-art methods without data rehearsal or additional communication, and even surpasses the jointly fine-tuned model, demonstrating its efficiency and robustness.
% \end{abstract}

\begin{abstract}

Real-world intelligent systems often require both distributed collaboration across data-isolated clients and continual adaptation to evolving tasks. This setting naturally gives rise to Federated Class Incremental Learning (FCIL), which combines Federated Learning (FL) and Continual Learning (CL). However, their combination introduces two coupled sources of interference: spatial interference from heterogeneous clients and temporal interference from sequential tasks, jointly leading to Spatial–Temporal Catastrophic Forgetting (ST-CF).
Existing approaches typically address spatial and temporal interference with separate mechanisms, often incurring additional client-side computation or communication, while leaving directional interactions among updates during aggregation unregulated.
In this paper, we reinterpret FCIL as a unified multi-task learning problem, where both client and task updates are represented as adaptation vectors in a shared parameter space. Based on this view, we propose \methodlong{} (\method{}), a purely server-side framework that performs geometric surgery on adaptation vectors during aggregation. Spatial \method{} mitigates client-level interference within each round, while causal online temporal \method{} removes cross-task interference over time without additional client-side computation, communication, or memory beyond standard federated training.
Empirically, \method{} achieves up to 22\% improvement over prior FCIL methods across diverse vision and language benchmarks while remaining robust to unreliable clients and maintaining computational efficiency.
\end{abstract}

\section{Introduction}
\begin{figure}
    \centering
    \includegraphics[width=0.9\linewidth]{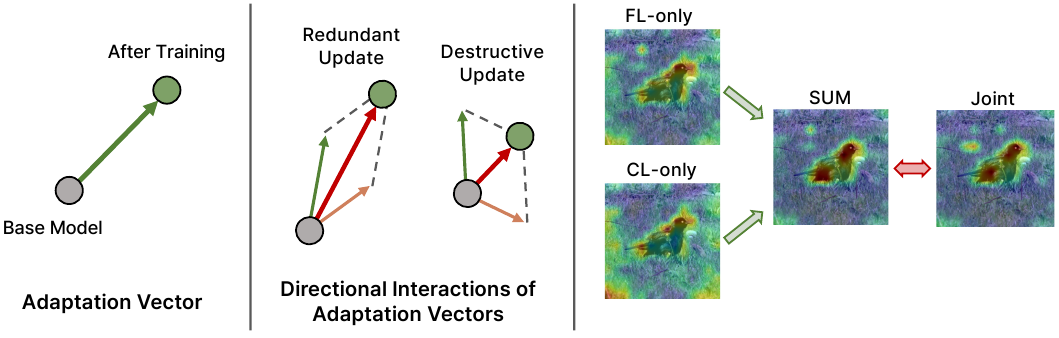}
\caption{
\textbf{Left:} Client or task updates during FCIL are represented as adaptation vectors in parameter space from the base model.
\textbf{Middle:} Directional interactions among these vectors may lead to redundant or destructive updates.
\textbf{Right:} Rather than resolving these interactions along a single axis (FL-only or CL-only), \method{} unifies both axes and regulates them, producing representations comparable to joint training.
}
    \label{fig:upload_cost}
\end{figure}
\label{sec:intro}
Federated Learning (FL) enables collaboration across distributed clients without sharing private data~\cite{fedavg,yang2023fedfed}, while Continual Learning (CL) allows models to adapt to evolving tasks over time~\cite{shin2017continual,ewc}. 
In many real-world settings, these capabilities are required simultaneously: clients must learn sequential tasks from heterogeneous, non-independent and identically distributed (non-IID) data streams while keeping data local. 
Federated Class-Incremental Learning (FCIL) therefore emerges as a practical paradigm for lifelong federated intelligence~\cite{target,dong2022federated}.

However, combining FL and CL introduces tightly coupled sources of interference:
\textit{spatial interference} caused by client drift under non-IID data~\cite{karimireddy2020scaffold}, 
and \textit{temporal interference} caused by catastrophic forgetting across tasks~\cite{mccloskey1989catastrophic}. 
These effects reinforce each other, forming spatial--temporal catastrophic forgetting (ST-CF)~\cite{fedta,fclsurvey}. 
Consequently, mitigating only one axis often leaves residual interference from the other.

During FCIL, updates from heterogeneous clients and sequential tasks are repeatedly combined in the shared parameter space. 
Their directional interactions may produce \emph{redundant} or \emph{destructive} updates (Fig.~\ref{fig:upload_cost}), 
leading to degraded optimization. 
Similar directional effects have been extensively studied in multi-task learning (MTL), where gradients from different tasks may point in incompatible directions, resulting in degraded generalization~\cite{senushkin2023independent,yu2020gradient,liu2021conflict,worsham2020multi}. 
However, existing FCIL methods mitigate ST-CF primarily through FL-oriented alignment~\cite{dong2022federated}, CL-style regularization~\cite{ewc}, replay~\cite{target,ma2022continual,MFCL}, or prototype anchoring~\cite{fedta,pilora}. 
Although some works address spatio-temporal interference jointly~\cite{fot, stamp, fedta, yu2024overcoming}, most methods regulate update generation during local optimization through constraints, auxiliary objectives, or data, while directional interactions among updates during aggregation remain largely unregulated.

In this paper, we revisit FCIL from a unified MTL perspective, where both client and task updates are represented as adaptation vectors—interpretable as gradient-like directions—in the shared parameter space. 
This view exposes their directional interactions explicitly and enables analysis of these interactions at aggregation time while preserving standard parameter communication~\cite{fedavg}. 
From this perspective, ST-CF can be understood as the cumulative effect of unresolved directional interactions: spatial drift introduces heterogeneous update directions across clients, while sequential tasks introduce interfering directions over time.

Building on this insight, we propose \methodlong{} (\method{}), a purely server-side aggregation-time refinement framework that performs geometric surgery on adaptation vectors before merging. 
A parallel spatial \method{} resolves client-level interactions within each round, while a causal online temporal \method{} regulates cross-task interactions over time. 
Importantly, \method{} operates entirely at the server during aggregation, introducing no additional client computation, communication, or memory beyond standard federated training. 
Extensive experiments on diverse vision and language benchmarks show that \method{} improves final averaged accuracy by up to 22\% over prior FCIL methods, remains robust under malicious-client settings, and maintains competitive computational overhead compared to existing FCIL approaches.

\begin{table*}[t]
\centering
\footnotesize
\setlength{\tabcolsep}{4pt}
\renewcommand{\arraystretch}{1.15}
\caption{
\textbf{Comparison with representative FCIL methods.}
}
\label{tab:fcl_comparison}

\adjustbox{max width=0.975\textwidth}{
\begin{tabular}{lcccc|cc}
\toprule
\textbf{Method}
& \multicolumn{4}{c|}{\textbf{Client-side overhead}}
& \multicolumn{2}{c}{\textbf{Property}} \\

\cmidrule(lr){2-5} \cmidrule(lr){6-7}

& \textbf{Extra objective}
& \textbf{Replay}
& \textbf{Extra memory}
& \textbf{Training modification}
& \textbf{Server-only}
& \textbf{Interaction modeling} \\

\midrule

GLFC~\cite{dong2022federated}
& \cmark & \cmark & \cmark & \cmark & \xmark & \xmark \\

TARGET~\cite{target}
& \cmark & \cmark & \cmark & \cmark & \xmark & \xmark \\

MFCL~\cite{MFCL}
& \cmark & \cmark & \xmark & \cmark & \xmark & \xmark \\

PILoRA~\cite{pilora}
& \xmark & \xmark & \cmark & \cmark & \xmark & \xmark \\

FOT~\cite{fot} 
& \xmark & \xmark & \cmark & \cmark & \xmark & \xmark \\

STAMP~\cite{stamp} 
& \cmark & \cmark & \cmark & \cmark & \xmark & \xmark \\

LoRM~\cite{lorm} 
& \xmark & \xmark & \cmark & \xmark & \xmark & \xmark \\

\midrule
\textbf{SUM (ours)} 
& \xmark & \xmark & \xmark & \xmark & \cmark & \cmark \\

\bottomrule
\end{tabular}}
\end{table*}

\section{Related Works}
\label{sec:rw} 

\paragraph{Federated Class Incremental Learning.}
Federated Class Incremental Learning (FCIL) combines Federated Learning (FL) and Continual Learning (CL), yet most existing methods address spatial drift and temporal forgetting separately. 
Methods such as GLFC~\cite{dong2022federated}, CFeD~\cite{ma2022continual}, and TARGET~\cite{target} extend FedAvg~\cite{fedavg} with rehearsal or distillation, while parameter-efficient variants (\eg FedWeIT~\cite{yoon2021federated}, PILoRA~\cite{pilora}, LoRM~\cite{lorm}, Fed-CPrompt~\cite{fedcprompt}) introduce task- or client-specific modules. 
Recent FCIL methods also use regularization, replay, stabilization, or distillation to preserve previous-task knowledge~\cite{MFCL,fedta,fedssi,refedplus,li2026feature}.
Other works, including FOT~\cite{fot} and STAMP~\cite{stamp}, reduce interference through projection or gradient-alignment mechanisms during local optimization: FOT constrains updates using activation-derived subspaces at task boundaries, while STAMP enforces gradient alignment during client training. 
As a result, these methods regulate update generation largely during local training or through auxiliary knowledge-preservation mechanisms, but do not explicitly model how client and task updates interact when accumulated on the server. 
In contrast, \method{} operates purely at aggregation time, treating updates as adaptation vectors and explicitly decomposing their directional interactions before merging. 
This design requires no modification to client-side training, since refinement is applied only to vectors produced by standard local training and communication. 
Tab.~\ref{tab:fcl_comparison} summarizes the key differences between \method{} and representative FCIL approaches.

\paragraph{Directional Gradient Interactions.}
In Multi-Task Learning (MTL), directional interactions among task gradients are a well-known source of redundant or destructive interference that can destabilize optimization and degrade generalization. 
Many works therefore mitigate such interference by explicitly manipulating gradients during optimization, including PCGrad~\cite{yu2020gradient}, MGDA~\cite{sener2018multi}, CAGrad~\cite{liu2021conflict}, GEM~\cite{lopez2017gradient}, Aligned-MTL~\cite{senushkin2023independent}, and GPM~\cite{saha2021gradient}. 
These methods stabilize training by projecting gradients or enforcing compatible descent directions across tasks.
In FCIL, STAMP~\cite{stamp} follows a similar philosophy by performing spatio-temporal gradient matching during client optimization.
In contrast, \method{} resolves interference after local training by refining the geometry of communicated adaptation vectors at server aggregation time, rather than relying on centralized optimization-time gradients or modifying client-side training. 
This preserves the standard federated communication and training protocol.

\paragraph{Adaptation Vectors and Model Merging.}
Model merging studies how independently trained task updates can be represented and composed in a shared parameter space. 
In this context, task-specific updates are often represented as \textit{adaptation vectors}, defined as parameter differences between a base model and a task-adapted model. 
Task Arithmetic~\cite{ilharco2022editing} introduced this formulation, later extended by TIES-Merging~\cite{yadav2024ties} and EMR-Merging~\cite{huang2024emr}, and applied to continual (\eg MagMax~\cite{magmax}) and federated settings (\eg FedAWA~\cite{fedawa}, FedMerge~\cite{fedmerge}). 
These approaches focus on parameter selection or weighted aggregation, but do not explicitly decompose directional interactions among coexisting adaptation vectors. 
In FCIL, LoRM~\cite{lorm} computes closed-form spatio-temporal merging weights, yet still treats task vectors primarily as additive components. 
In contrast, \method{} models both client- and task-induced updates as interacting adaptation vectors, explicitly handling within-round client coupling and cross-task accumulation through a unified server-side refinement protocol.

\section{Foundations}
\label{sec:foundations}
\subsection{FCIL as a Unified Multi-Task Learning Problem}

We view Federated Class Incremental Learning (FCIL) through a multi-task learning (MTL) perspective.
Federated Learning corresponds to \emph{spatial MTL}, where heterogeneous clients optimize distinct objectives under shared parameters, while Continual Learning corresponds to \emph{temporal MTL}, where tasks arrive sequentially under the same model~\cite{magmax,saha2021gradient}.
FCIL naturally combines both structures: objectives vary across clients and evolve across tasks within a shared parameter space.
A formal joint objective and its equivalence to this unified MTL view are provided in Appendix~\ref{sup:fcil_mtl}.
Under this view, interference arises when updates from different objectives interact through aggregation, which prior MTL studies attribute largely to directional relationships among gradients~\cite{yu2020gradient,liu2021conflict}.

\subsection{Adaptation Vectors in a Shared Parameter Space}
\label{subsec:gradient_view_FCL}

To formalize the directional perspective inspired by MTL, we analyze update interactions directly in the parameter space. 
In FCIL, clients exchange model parameters rather than gradients, making the parameter space the natural domain for analyzing update interactions.
We therefore represent both client- and task-induced updates as \emph{adaptation vectors} in the parameter space of the shared model. 
Intuitively, an adaptation vector corresponds to the parameter change between two model states (\eg before and after training)~\cite{ilharco2022editing,yadav2024ties}. 
Since these updates are produced by gradient-based optimization, they capture the accumulated effect of gradient directions over local training steps or sequential tasks (Appendix~\ref{sup:proof_prop1}). 
This allows directional interactions to be analyzed through the resulting adaptation vectors, which are actually merged during federated and continual aggregation.
We first define the client and task adaptation vectors.

\paragraph{Client Adaptation (Spatial Axis).}
At communication round $t$, the server broadcasts $\theta_G^t$.  
Each client $i$ performs $K_i$ steps of local training and returns the updated parameters $\theta_i^{t,K_i}$.
We define the client adaptation vector $\tau_i^t$ as:
\begin{equation}
\tau_i^t := \theta_i^{t,K_i} - \theta_G^t.
\label{eq:client_vector_def}
\end{equation}
Using this adaptation vector, standard FedAvg~\cite{fedavg} can be written as
\begin{equation}
\theta_G^{t+1}
= \theta_G^t + \sum_{i \in \mathcal{S}_t} \lambda_i \tau_i^t.
\label{eq:fedavg}
\end{equation}

\paragraph{Task Adaptation (Temporal Axis).}
Let $\theta_G^{(k)}$ denote the global model obtained after completing training rounds on task $k$. 
Across tasks, we define an \emph{accumulated} task adaptation vector $\tau_k$, relative to a common pretrained base $\theta_{\text{pre}}$:
\begin{equation}
\tau_k := \theta_G^{(k)} - \theta_{\text{pre}}.
\label{eq:task_vector_def}
\end{equation}
Rather than measuring updates relative to the immediately preceding model, we anchor all tasks to the same pretrained base. 
This shared reference places adaptation vectors in a common parameter frame, which both satisfies the conditions of our theoretical analysis (Thm.~\ref{thm:main}) and enables explicit regulation of their directional interactions.
As we show in Sec.~\ref{subsec:temporal_sum}, operating on these accumulated vectors is equivalent to resolving the interactions among the underlying true sequential task updates.

\paragraph{Unifying Perspective.}
Both client updates $\tau_i^t$ and task updates $\tau_k$ reside in the same shared parameter space and accumulate through repeated aggregation.
Their directional interactions therefore shape the evolution of the global model.
Across communication rounds, heterogeneous clients introduce diverse update directions, while sequential tasks introduce additional directions over time.
From MTL perspective, when such directional interactions are not explicitly regulated, their accumulated effects can distort shared representations~\cite{yu2020gradient,liu2021conflict}, providing a geometric view of spatial--temporal catastrophic forgetting (ST-CF).
This perspective motivates our proposed \method{} introduced in Sec.~\ref{sec:method}.

\section{Method}
\label{sec:method}

Building on Sec.~\ref{sec:foundations}, we present a unified framework for resolving spatio-temporal interference in Federated Class Incremental Learning (FCIL).  
We first introduce our unified projection principle.
% ---------------------------------------------------------
\begin{figure*}[t]
    \centering    \includegraphics[width=\linewidth]{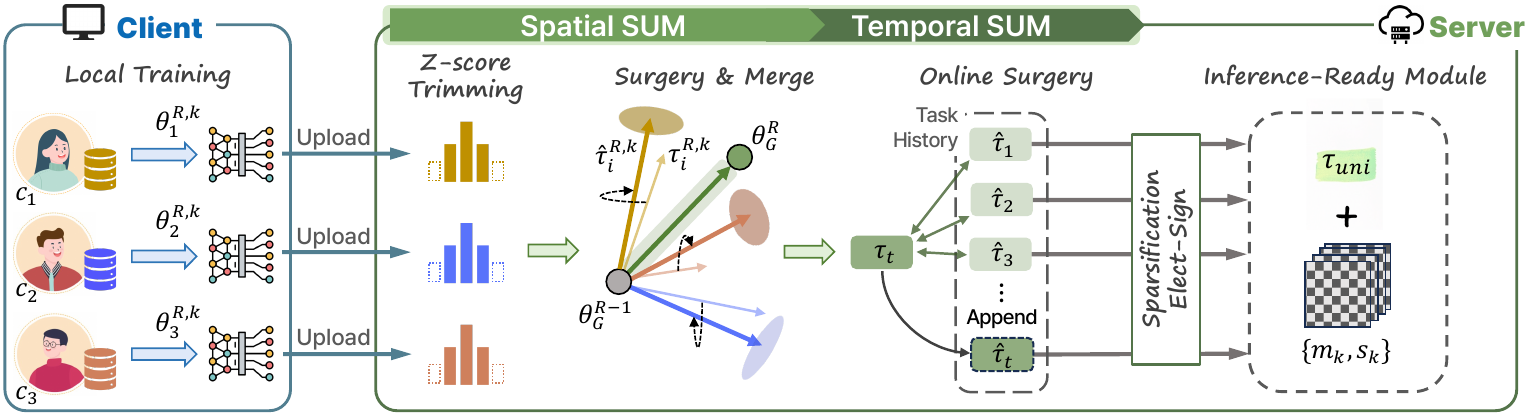}
\caption{\textbf{Overview of \methodlong{} (\method{})}. 
After local training for task $k$ over $R$ communication rounds, clients upload their updated weights to the server. 
The server then applies spatial and temporal \method{} to refine the adaptation vectors and construct an inference-ready module for task $k$. 
Notably, \method{} introduces no additional client-side overhead beyond standard federated model training and communication.}
    \label{fig:main}
\end{figure*}

\subsection{\methodlong{}}
\label{subsec:sum_principle}

% To resolve the directional interactions identified in Sec.~\ref{subsec:gradient_view_FCL}, 
% we propose \methodlong{} (\method{}), 
% a refinement framework applied at aggregation time.

\paragraph{Overall Pipeline.}
As shown in Fig.~\ref{fig:main}, \method{} operates in three stages across tasks.
The server applies Spatial SUM to remove directional interference among client adaptation vectors, and Temporal SUM to regulate interactions with previously learned tasks.
The resulting task vectors are stored as a temporal basis and converted into a compact inference module after each task.
The resulting module is used for inference, while subsequent local training proceeds from the spatially aggregated model.

\paragraph{Core Idea.}
Before introducing our refinement, we denote the standard aggregation result as
\begin{equation}
\theta_{\text{standard}}
=
\theta_{\text{base}}
+
\sum_{i} \lambda_i \mathbf{v}_i ,
\label{eq:standard_merge}
\end{equation}
where $\theta_{\text{base}}$ denotes the anchor model (\eg $\theta_G^t$ or $\theta_{\text{pre}}$) and $\mathbf{v}_i$ denotes an adaptation vector produced during training (\eg a client update $\tau_i^t$ or a task update $\tau_k$).
In FCIL, the server aggregates multiple adaptation vectors produced by heterogeneous clients or sequential tasks using Eq.~\ref{eq:standard_merge}.
When these vectors are combined through naive summation, their directional interactions may produce \emph{redundant} updates or \emph{destructive} updates, which distort the resulting representation.

Our goal is therefore to regulate these directional interactions during aggregation.
Given adaptation vectors $\{\mathbf{v}_i\}$, 
\method{} refines each vector by removing components aligned with others, retaining only its independent contribution in the shared parameter space.
Formally, we compute a refined vector
\begin{equation}
\hat{\mathbf{v}}_i 
= \mathbf{v}_i 
- \sum_{\substack{j=1 \\ j \neq i}}^{n} 
\frac{\mathbf{v}_i \cdot \mathbf{v}_j}{\|\mathbf{v}_j\|_2^2} \, \mathbf{v}_j .
\label{eq:general_projection}
\end{equation}
The resulting model is then obtained as
\begin{equation}
\theta_{\method{}} 
= \theta_{\text{base}} 
+ \sum_{i=1}^{n} \lambda_i \hat{\mathbf{v}}_i .
\label{eq:general_merge}
\end{equation}
Intuitively, this refinement reduces directional interference among adaptation vectors, preventing overlapping or conflicting updates from dominating the aggregated model.
Importantly, \method{} does not discard entire client or task updates; it removes only the components that are directionally coupled with other vectors, thereby preserving residual client- or task-specific directions.
This interpretation is empirically supported by our surgery-target ablation in Fig.~\ref{fig:surgery_ablation}, where one-sided suppression of selected interactions is insufficient, whereas \method{} preserves useful residual directions and achieves the best performance.
Because the refinement operates directly in the parameter space, 
\method{} can be applied once at aggregation on the server side without modifying client optimization.
Unlike optimization-time methods~\cite{target,dong2022federated,fot}, 
it requires no gradient storage or additional forward/backward passes.

We further provide a local first-order analysis showing when \method{} can
improve the aggregation-level training descent bound under standard smoothness
assumptions~\cite{yu2020gradient,liu2021conflict}.

\begin{theorem}
Let $\theta_{\method{}}$ (Eq.~\ref{eq:general_merge}) and $\theta_{\text{standard}}$ (Eq.~\ref{eq:standard_merge}) be models obtained from refined and standard merging, respectively.  
If all $\mathbf{v}_i$ share the common base $\theta_{\text{base}}$ and task losses are convex with $L$-Lipschitz gradients at $\theta_{\text{base}}$, then:
\[
\mathcal{L}(\theta_{\method{}}) 
\leq 
\mathcal{L}(\theta_{\text{standard}}).
\]
\label{thm:main}
\end{theorem}
See Supplementary Sec.~\ref{app:proof_main} for the proof.
This result should be interpreted as an aggregation-level local training-objective
statement under the stated common-base and smoothness assumptions, rather than
as a universal test-loss or global convergence guarantee. It formalizes a sufficient
condition under which reducing directional interference yields a tighter local
descent bound than naive aggregation. Generalization behavior is therefore
evaluated empirically through held-out Final Averaged Accuracy (FAA),
forgetting, sharpness, and class-margin analyses (Sec.~\ref{sec:exp}, Appendix~\ref{sup:additional_experiments}).

% ---------------------------------------------------------
\subsection{Spatial \method{}}
\label{subsec:spatial_sum}

In FL, client adaptation vectors $\tau_i^t$ coexist within each communication round and often diverge due to non-IID data.  
Spatial \method{} instantiates Eq.~\ref{eq:general_merge} on the set of client adaptation vectors at round $t$:
\begin{equation}
\theta_{G}^{t+1}
=
\theta_{G}^{t}
+
\lambda_S \sum_{i \in \mathcal{S}} \hat{\tau}_i^t,
\label{eq:spatial_sum}
\end{equation}
where $\lambda_S$ is a scaling factor for spatial aggregation. Before applying Eq.~\ref{eq:spatial_sum} to the client adaptation vectors, 
we suppress extreme coordinates via Z-score trimming. 
For each coordinate index $p$, we zero entries with $|z[p]| > z_{\text{thr}}$, 
where $z[p] = (\tau[p]-\mu)/\sigma$ denotes the Z-score of the $p$-th parameter coordinate.
Without this step, outlier coordinates may disproportionately dominate the inner products used in Eq.~\ref{eq:general_projection}. 
This can make the projection excessively large or even flip its direction, resulting in unstable refinement updates (Tab.~\ref{tab:ablation}, Appendix~\ref{supsec:z_score_trimming}).
Classifier heads are merged using RegMean~\cite{jin2022dataless}.  

\paragraph{System and Overhead.}
Clients perform standard local training and upload model updates exactly as in conventional FL; all \method{} operations are executed on the server, introducing no additional client-side computation or communication.

% ---------------------------------------------------------
\subsection{Temporal \method{}}
\label{subsec:temporal_sum}

After the final FL communication round $R$ for task $k$, the server obtains $\theta_G^{R}=\theta_G^{(k)}$ via the Spatial \method{} and proceeds to Temporal \method{}.  
Let the true sequential updates be
\(
\Delta_k = \theta_G^{(k)} - \theta_G^{(k-1)}.
\)
Although $\{\Delta_k\}$ capture the true learning dynamics of task $k$, each update is defined relative to a different base model $\theta_G^{(k-1)}$. 
Applying Eq.~\ref{eq:general_projection} directly to these updates therefore violates the common-base assumption required by Thm.~\ref{thm:main}, which assumes that all vectors are defined relative to a shared base model $\theta_{\text{base}}$. 
To satisfy this condition, we re-anchor task adaptations to a common base model:
\begin{equation}
\tau_k = \theta_G^{(k)} - \theta_{\text{pre}}.
\label{eq:task_vector_def_re}
\end{equation}

\paragraph{Online Surgery.}
As tasks arrive sequentially in FCIL, we perform an online projection that incrementally handles directional interactions along previously stored task adaptation vectors.
\begin{equation}
\hat{\tau}_k 
=
\tau_k 
-
\sum_{j=1}^{k-1}
\frac{\tau_k \cdot \hat{\tau}_j}{\|\hat{\tau}_j\|_2^2}
\hat{\tau}_j,
\quad
(\hat{\tau}_1 = \tau_1).
\label{eq:temporal_projection}
\end{equation}
This online form is a sequential implementation of the same projection principle used in Eq.~\ref{eq:general_projection}. 
As discussed in Appendix~\ref{app:proof_main_temporal}, it admits the same local descent-bound
interpretation under the stated assumptions.
The following proposition shows that performing projection on the accumulated vectors $\{\tau_k\}$ is equivalent to resolving directional interactions among the true sequential task updates $\{\Delta_k\}$ (See Appendix~\ref{app:proof_prop_temporal} for the proof).

\begin{proposition}
Applying online surgery to $\{\tau_k\}$ produces the identical orthogonal vectors as applying the same procedure directly to $\{\Delta_k\}$.
\label{prop:temporal_equivalence}
\end{proposition} 
% The refined vectors $\{\hat{\tau}_k\}$ form a compact temporal basis $\mathcal{B}$.

\paragraph{System and Overhead.}
Temporal \method{} is executed entirely on the server using the uploaded model weights after each task. 
The server computes $\tau_k$, performs the projection, and stores the vectors $\{\hat{\tau}_k\}$.

% ---------------------------------------------------------
\subsection{Inference-ready Module Construction}
\label{subsec:deployment}

After temporal \method{}, each $\hat{\tau}_k$ captures the interference-free adaptation of task $k$.  
While client training continues from the global model $\theta_G^{(k)}$ for the next task, we construct compact inference modules for deployment.  
Specifically, we convert $\{\hat{\tau}_k\}$ into lightweight task-specific modules using a modified EMR-Merging~\cite{huang2024emr}.

\paragraph{Step 1: Sparsification.}
Each $\hat{\tau}_k$ is sparsified via magnitude-based top-$k$ selection, retaining the largest $k_{pct}$ fraction of coordinates. 
Let $\Omega_k$ denote the set of indices corresponding to the largest $k_{pct}$ fraction of $|\hat{\tau}_k|$:
\[
\bar{\tau}_k[p] =
\begin{cases}
\hat{\tau}_k[p], & p \in \Omega_k, \\
0, & \text{otherwise}.
\end{cases}
\]
In practice, we find that retaining only these high-magnitude coordinates is sufficient to preserve the dominant task-specific adaptation patterns (Fig.~\ref{fig:param_sensitivity}).

\paragraph{Step 2: Elect-Sign.}
From $\bar{\mathcal{B}}=\{\bar{\tau}_1,\ldots,\bar{\tau}_k\}$, we derive a unified direction $\tau_{\text{uni}}$ using a coordinate-wise sign-consensus rule.  
For each parameter coordinate, we determine the majority-supported sign across tasks and assign its magnitude using the largest aligned value.  
Coordinates whose signs disagree across tasks are set to zero.  
This retains directions that are consistently aligned across tasks while removing conflicting updates.

\paragraph{Step 3: Task-Specific Activation.}
To preserve task-specific behavior, each task activates a subset of the unified direction.  
We construct a binary mask $m_k$ where $m_k[p]=1$ if $\tau_k[p]\tau_{\text{uni}}[p] > 0$ and $0$ otherwise.  
A scalar factor
$
s_k = \frac{\|\tau_k\|_1}{\|m_k \odot \tau_{\text{uni}}\|_1 + \varepsilon}
$
rescales the activation to match the original task magnitude, where $\varepsilon$ is a small constant for numerical stability.
The inference model is reconstructed as:
$
\theta_{\text{infer}}^{(k)} 
= \theta_{\text{pre}} + m_k \odot (s_k \cdot \tau_{\text{uni}}).
$
The server distributes the compact module $(\tau_{\text{uni}}, m_k, s_k)$, while the shared pretrained backbone $\theta_{\text{pre}}$ is already available on clients.

\section{Experiments}
\label{sec:exp}

% In this section, we evaluate \method{} across domains, datasets, and backbones, compare it with state-of-the-art baselines, and analyze its components and learning dynamics.

\subsection{Experimental Settings}
\label{subsec:settings}

We follow a standard FCIL setup~\cite{lorm,hgp}, using the same pure ImageNet-pretrained ViT-B/16~\cite{vit} backbone for vision tasks (Sec.~\ref{subsec:vision_domain}) and T5-Small~\cite{t5} for language tasks (Sec.~\ref{subsec:language_domain}). 
Unless stated otherwise, experiments use 10 clients, 10 sequential class-split tasks, and 5 communication rounds per task. 
Client data follow a Dirichlet label-imbalance partition with $\beta \in \{0.5,0.1,0.05\}$ or $\{1.0,0.5,0.2\}$ depending on dataset scale. 
We report Final Averaged Accuracy (FAA) after the last task. 
For \method{}, we set $z_{\text{thr}}=4.5$, $\lambda_S=0.4$, and $k_{pct}=5\%$. 
A centralized \textit{Joint} baseline is included.
All baselines are evaluated under the same FCIL protocol~\cite{lorm}, using the same backbone, task splits, client partitions, communication rounds, and data-access constraints.
Method-specific hyperparameters are selected following the original recommendations and reported in Appendix ~\ref{supsec:hyperparameters}.

\subsection{Vision Domain}
\label{subsec:vision_domain}

\renewcommand{\arraystretch}{1.1}
\begin{table*}[t]
\caption{\textbf{Evaluation on various vision-domain datasets in FAA ($\mathbf{\uparrow}$).}
\textbf{Bold} indicates the best and \underline{underlined} the second-best results.}
\centering
\setlength{\tabcolsep}{0.15em}
\adjustbox{max width=\textwidth}{
\begin{tabular}{
@{}l|
ccc|ccc|ccc|ccc|ccc|ccc@{}
}
\toprule
& \multicolumn{3}{c|}{\textbf{CIFAR-100}} &
  \multicolumn{3}{c|}{\textbf{ImageNet-R}} &
  \multicolumn{3}{c|}{\textbf{ImageNet-A}} &
  \multicolumn{3}{c|}{\textbf{EuroSAT}} &
  \multicolumn{3}{c|}{\textbf{CARS-196}} &
  \multicolumn{3}{c}{\textbf{CUB-200}} \\
\cmidrule(lr){2-19}
\textbf{Joint} &
\multicolumn{3}{c|}{$\boldsymbol{92.75}$} &
\multicolumn{3}{c|}{$\boldsymbol{84.02}$} &
\multicolumn{3}{c|}{$\boldsymbol{54.64}$} &
\multicolumn{3}{c|}{$\boldsymbol{98.42}$} &
\multicolumn{3}{c|}{$\boldsymbol{85.62}$} &
\multicolumn{3}{c}{$\boldsymbol{86.04}$} \\
\midrule
\textbf{Distrib. $\boldsymbol{\beta}$} &
$0.5$ & $0.1$ & $0.05$ &
$0.5$ & $0.1$ & $0.05$ &
$1.0$ & $0.5$ & $0.2$ &
$1.0$ & $0.5$ & $0.2$ &
$1.0$ & $0.5$ & $0.2$ &
$1.0$ & $0.5$ & $0.2$ \\
\midrule
EWC & $78.46$ & $72.42$ & $64.51$ & $58.93$ & $48.15$ & $43.68$ & $10.86$ & $10.07$ & $8.89$ &
$64.12$ & $59.30$ & $56.52$ & $19.55$ & $18.02$ & $18.29$ & $31.46$ & $29.60$ & $27.89$ \\

LwF & $62.87$ & $55.56$ & $47.09$ & $54.03$ & $41.02$ & $46.07$ & $8.89$ & $8.89$ & $7.90$ &
$31.91$ & $21.26$ & $31.42$ & $20.84$ & $22.72$ & $31.76$ & $25.25$ & $21.11$ & $18.54$ \\

FisherAvg & $76.10$ & $74.43$ & $65.31$ & $58.68$ & $50.82$ & $47.33$ & $11.59$ & $11.06$ & $10.14$ &
$58.84$ & $59.94$ & $55.86$ & $26.03$ & $24.60$ & $21.58$ & $30.45$ & $28.39$ & $25.06$ \\

RegMean & $59.80$ & $45.88$ & $39.08$ & $61.18$ & $57.00$ & $55.80$ & $8.56$ & $6.22$ & $4.34$ &
$48.74$ & $51.73$ & $45.27$ & $21.83$ & $20.36$ & $15.92$ & $35.57$ & $32.84$ & $32.83$ \\

CCVR & $79.95$ & $75.14$ & $65.30$ & $70.00$ & $62.60$ & $60.38$ & \underline{$39.50$} & $36.27$ & \underline{$35.94$} &
$64.44$ & $57.93$ & $62.69$ & $38.99$ & $37.81$ & $35.31$ & $62.67$ & $59.48$ & $56.33$ \\

L2P & $83.88$ & $61.54$ & $55.00$ & $42.08$ & $23.85$ & $16.98$ & $20.14$ & $17.31$ & $16.85$ &
$40.63$ & $51.78$ & $45.46$ & $35.49$ & $31.00$ & $20.01$ & $56.23$ & $47.31$ & $38.16$ \\

CODA-P & $82.25$ & $61.82$ & $46.74$ & $61.18$ & $36.73$ & $25.82$ & $18.30$ & $14.48$ & $7.31$ &
$73.38$ & $69.42$ & $66.69$ & $28.04$ & $20.83$ & $14.53$ & $42.53$ & $37.71$ & $29.19$ \\

FedProto & $75.79$ & $70.02$ & $60.55$ & $58.52$ & $47.30$ & $52.93$ & $9.87$ & $9.22$ & $10.01$ &
$58.79$ & $62.85$ & $64.17$ & $26.08$ & $24.55$ & $22.75$ & $30.22$ & $28.27$ & $26.01$ \\

TARGET & $74.72$ & $72.32$ & $62.60$ & $54.65$ & $45.83$ & $41.32$ & $10.27$ & $11.39$ & $10.73$ &
$52.74$ & $52.74$ & $45.11$ & $28.65$ & $27.20$ & $26.13$ & $39.30$ & $38.40$ & $34.79$ \\

PILoRA & $76.48$ & $75.81$ & $74.80$ & $53.67$ & $51.62$ & $49.37$ & $19.62$ & $18.70$ & $20.01$ &
$48.35$ & $32.89$ & $31.22$ & $37.57$ & $37.92$ & $36.95$ & $61.11$ & $60.68$ & \underline{$60.39$} \\

FOT & $79.26$ & $76.45$ & $69.89$ & $68.65$ & $59.92$ & $55.12$ & $18.37$ & $17.25$ & $15.01$ &
$59.58$ & $57.07$ & $48.07$ & $33.53$ & $33.49$ & $30.57$ & $45.88$ & $41.78$ & $41.28$ \\

LoRM & \underline{$86.95$} & \underline{$81.75$} & \underline{$82.76$} & \underline{$72.48$} & \underline{$63.83$} & \underline{$66.45$} & $37.26$ & \underline{$36.34$} & $33.11$ &
\underline{$84.23$} & \underline{$77.26$} & \underline{$81.36$} & \underline{$54.41$} & \underline{$51.87$} & \underline{$48.81$} & \underline{$64.60$} & \underline{$63.67$} & $60.06$ \\
\midrule
\rowcolor{gray!10} \textbf{\method{}} \hspace{-1em} &
$\boldsymbol{97.03}$ & $\boldsymbol{96.93}$ & $\boldsymbol{96.32}$ &
$\boldsymbol{83.80}$ & $\boldsymbol{85.58}$ & $\boldsymbol{86.81}$ &
$\boldsymbol{49.90}$ & $\boldsymbol{50.82}$ & $\boldsymbol{47.86}$ &
$\boldsymbol{96.51}$ & $\boldsymbol{93.28}$ & $\boldsymbol{98.57}$ &
$\boldsymbol{68.78}$ & $\boldsymbol{68.83}$ & $\boldsymbol{67.57}$ &
$\boldsymbol{81.69}$ & $\boldsymbol{79.79}$ & $\boldsymbol{79.43}$ \\
\bottomrule
\end{tabular}}
\label{tab:combined_eval}
\end{table*}

\paragraph{Datasets.}
% We evaluate \method{} on six widely used datasets spanning different levels of domain proximity to ImageNet pre-training.
% Following prior work~\cite{lorm}, we group them into three categories.
% \textit{In-domain datasets} include CIFAR-100~\cite{cifar100}, ImageNet-R~\cite{imagenetr}, and ImageNet-A~\cite{imageneta}, which remain close to the natural image distribution of ImageNet.
% \textit{Fine-grained datasets} consist of Cars-196~\cite{krause2013cars} and CUB-200~\cite{cub200}, focusing on category-level specialization within a narrow domain.
% Finally, the \textit{out-of-domain dataset}, EuroSAT~\cite{helber2019eurosat}, contains satellite imagery that differs substantially from the pre-training distribution.
% All datasets are split into class-incremental tasks (10 for most datasets, 5 for EuroSAT).
% For ImageNet-A and Cars-196, we increase communication rounds to 10 for convergence stability, following prior work~\cite{lorm,hgp}.
% Additional preprocessing and task construction details are provided in Appendix Sec.~\ref{supsec:dataset_details}.

We evaluate \method{} on six datasets spanning varying proximity to ImageNet. 
Following prior work~\cite{lorm}, we group them into three categories: 
in-domain (CIFAR-100~\cite{cifar100}, ImageNet-R~\cite{imagenetr}, ImageNet-A~\cite{imageneta}), 
fine-grained (Cars-196~\cite{krause2013cars}, CUB-200~\cite{cub200}), 
and out-of-domain (EuroSAT~\cite{helber2019eurosat}). 
All datasets are split into class-incremental tasks (10 for most, 5 for EuroSAT), 
and ImageNet-A and Cars-196 use 10 FL rounds for stable convergence~\cite{lorm,hgp}. 
% Additional task details are in Appendix Sec.~\ref{supsec:dataset_details}.

\paragraph{Baselines.}
We compare \method{} with representative methods from FL, CL, model merging, and FCIL.
From FL, we include CCVR~\cite{ccvr} and FedProto~\cite{fedproto}, extended to continual learning using ACE~\cite{caccia2021new}. 
From CL, we evaluate EWC~\cite{ewc}, LwF~\cite{lwf}, L2P~\cite{l2p}, and CODA-Prompt~\cite{codaprompt}, adapted to the federated setting via FedAvg~\cite{fedavg}. 
From model merging, we include FisherAvg~\cite{matena2022merging} and RegMean~\cite{jin2022dataless}. 
Finally, we compare with FCIL-specific methods TARGET~\cite{target}, PILoRA~\cite{pilora}, FOT~\cite{fot}, and LoRM~\cite{lorm}. 
For fairness, we exclude approaches requiring surrogate or centralized server data~\cite{fedta,ma2022continual,surrdata1}.
\paragraph{Results.}
As shown in Tab.~\ref{tab:combined_eval}, \method{} achieves the highest FAA across all task settings, outperforming all baselines by a large margin.
The improvement is consistent across all three dataset categories (\ie in-domain, fine-grained, and out-of-domain), demonstrating the robustness of spatio-temporal refinement of \method{}.
Notably, on several datasets such as CIFAR-100, ImageNet-R, and EuroSAT, \method{} even surpasses \textit{Joint}. Sec.~\ref{subsec:analysis} analyzes this behavior through optimization sharpness and representation separation.

\subsection{Language Domain}
\label{subsec:language_domain}

\paragraph{Datasets.}
To verify that \method{} generalizes beyond vision, we further evaluate it on two text classification datasets: 20-Newsgroups~\cite{20nng} for in-domain performance and CLINC-150 (Out-of-Scope)~\cite{clinc150} for robustness under domain shift.
% Additional details are provided in Appendix Sec.~\ref{supsec:dataset_details}.

% \input{tables/text_2datasets}
% \input{tables/architecture}

\begin{table*}[t]
\centering

\begin{minipage}[b]{0.475\textwidth}
\renewcommand{\arraystretch}{0.7}
\footnotesize
\centering
\caption{\textbf{Evaluation on language-domain datasets in FAA ($\uparrow$).}}
\setlength{\tabcolsep}{0.2em}
\adjustbox{max width=\textwidth}{
\begin{tabular}{@{}l|ccc|ccc@{}}
\toprule
& \multicolumn{3}{c|}{\textbf{20-NewsGroups}} & \multicolumn{3}{c}{\textbf{CLINC-150}} \\
\cmidrule(lr){2-4}\cmidrule(lr){5-7}
\textbf{Joint} & \multicolumn{3}{c|}{$\boldsymbol{93.37}$} & \multicolumn{3}{c}{$\boldsymbol{95.53}$} \\
\midrule
\textbf{Distrib.} $\boldsymbol{\beta}$ & $1.0$ & $0.5$ & $0.2$ & $1.0$ & $0.5$ & $0.2$ \\
\midrule
RegMean     & $52.38$  & $50.96$ & $58.31$ & $74.16$ & $74.80$ & $70.78$ \\
CCVR        & $72.84$ & ${77.48}$ & $69.72$ & $39.18$ & $38.13$ & $30.36$ \\
CODA-P      & $68.71$ & $57.66$ & $56.35$ & $69.02$ & $59.40$ & $38.67$ \\
LoRM        & $\underline{88.58}$ & $\underline{84.73}$ & $\underline{88.10}$ & $\underline{84.78}$ & $\underline{82.58}$ & $\underline{77.67}$ \\
\midrule
\rowcolor{gray!10} 
\textbf{\method{}} & $\boldsymbol{89.42}$ & $\boldsymbol{85.53}$ & $\boldsymbol{89.52}$ 
& $\boldsymbol{96.24}$ & $\boldsymbol{95.93}$ & $\boldsymbol{94.11}$ \\
\bottomrule
\end{tabular}}
\label{tab:language_results}
\end{minipage}
\hfill
\begin{minipage}[b]{0.49\textwidth}
\renewcommand{\arraystretch}{1.0}
\footnotesize
\centering
\caption{\textbf{Backbone-wise robustness evaluation of \method{}.}}
\setlength{\tabcolsep}{0.8em}
\adjustbox{max width=\textwidth}{
\begin{tabular}{l l c c c}
\toprule
\multirow{2}{*}{\textbf{Backbone}} & \multirow{2}{*}{\textbf{Method}} &
\multicolumn{3}{c}{\textbf{Distrib. $\beta$}} \\
\cmidrule(lr){3-5}
 & & 0.5 & 0.1 & 0.05 \\
\midrule
\multirow{2}{*}{\textbf{LoRA}}
 & Joint & \multicolumn{3}{c}{92.93} \\
 & \g{\textbf{\method{}}} 
   & \g{\textbf{96.40}} 
   & \g{\textbf{95.50}} 
   & \g{\textbf{96.35}} \\
\midrule
\multirow{2}{*}{\textbf{ConvNeXt}}
 & Joint & \multicolumn{3}{c}{91.46} \\
 & \g{\textbf{\method{}}} 
   & \g{\textbf{96.19}} 
   & \g{\textbf{93.75}} 
   & \g{\textbf{93.93}} \\
\bottomrule
\end{tabular}}
\label{tab:architecture_robustness}
\end{minipage}

\end{table*}
\renewcommand{\arraystretch}{0.85}
\begin{table}[t]
\footnotesize
\centering
\caption{\textbf{Ablation study on the components of \method{} on CIFAR-100 ($\beta=0.05$)}}
\label{tab:ablation}
\setlength{\tabcolsep}{0.85em}

\adjustbox{max width=0.75\linewidth}{
\begin{tabular}{lccc}
\toprule
\textbf{} & \textbf{Ablation Setting} & \textbf{FAA ($\uparrow$)} & \textbf{Wall Time} \\
\midrule

\multirow{2}{*}{\textbf{Spatial \method{}}} 
& {\small w/o} Z-score Trimming & 85.87 {\scriptsize{\color{red}($\downarrow$10.45)}} & 1.00$\times$\\
& {\small w/o} Spatial Surgery   & 78.93 {\scriptsize{\color{red}($\downarrow$17.39)}} & 0.77$\times$\\
\midrule

\multirow{1}{*}{\textbf{Temporal \method{}}} 
& {\small w/o} Temporal Surgery  & 87.67 {\scriptsize{\color{red}($\downarrow$8.65)}} & 0.90$\times$\\
\midrule

\multirow{2}{*}{\textbf{Inference-ready Module}} 
& {\small w/o} Sparsification    & 90.23 {\scriptsize{\color{red}($\downarrow$6.09)}} & 0.92$\times$\\
& {\small w/o} Elect-Sign    & 94.58 {\scriptsize{\color{red}($\downarrow$1.74)}} & 0.93$\times$\\
& {\small w/o} Mask Modular      & 91.24 {\scriptsize{\color{red}($\downarrow$5.08)}} & 1.00$\times$\\
\midrule

\textbf{\method{}} & -- & 96.32 & 1.00$\times$\\
\bottomrule
\end{tabular}
}
\end{table}

\paragraph{Baselines.}
We adopt four representative baselines covering the paradigms of FL, CL, model merging, and FCIL: CCVR, CODA-Prompt, RegMean, and LoRM.

\paragraph{Results.}
Tab.~\ref{tab:language_results} presents the results of \method{} on the language domain.
\method{} retains its advantage beyond vision, achieving the highest FAA across all distribution levels.
It consistently outperforms all baselines and, in many cases, even surpasses \textit{Joint} under severe distribution shifts.
These results show that \method{} generalizes seamlessly from vision to language, 
maintaining stable optimization and consistent representations across heterogeneous domains.

\subsection{Architectural Robustness}
\label{subsec:architecture_robustness}
We assess the architectural robustness of \method{} on CIFAR-100 under two representative settings:
(1) a parameter-efficient backbone (LoRA~\cite{hu2022lora}) commonly used in FCIL, and 
(2) a CNN backbone (ConvNeXt~\cite{convnext}) to test generality beyond ViTs.
As shown in Tab.~\ref{tab:architecture_robustness}, \method{} maintains strong performance across both architectures and varying levels of client heterogeneity.
Notably, the LoRA-based \method{} even surpasses the full-finetuning (Tab.~\ref{tab:combined_eval}), consistent with findings that parameter-efficient modules preserve pretrained representations better than full fine-tuning while reducing communication cost~\cite{pilora, lorm}.
Overall, these results indicate that \method{} is architecture-agnostic and applicable across diverse models.

\subsection{Ablation Study}
\label{subsec:ablation}

% To examine the core mechanisms of \method{}, we perform ablation studies, analyzing the contribution of each component and the effect of surgery directions.

% \paragraph{\method{} Components.}
% In Tab.~\ref{tab:ablation}, we remove each component along the spatial and temporal axes to assess its contribution.
% Both spatial and temporal surgery (\ie Eq.~\ref{eq:general_projection} and Eq.~\ref{eq:temporal_projection}) constitute the core of the framework, as removing either leads to a significant performance drop, confirming their essential role in mitigating client- and task-level interference.
% Auxiliary components such as Z-score trimming, sparsification, and mask modularization further enhance stability and reconstruction fidelity, ensuring consistent performance under non-IID conditions and continual updates.
% Overall, while these supporting modules provide additional robustness, the spatial–temporal core of \method{} remains the primary driver of continual improvement.

\paragraph{\method{} Components.}
Tab.~\ref{tab:ablation} evaluates the contribution of each component of \method{} by removing them from the spatial and temporal axes as well as the inference-ready module.
The results show that spatial and temporal \method{} form the core of the framework, as removing either leads to a large performance drop. The inference module components are also non-negligible: removing sparsification, Elect-Sign, or task-specific activation consistently degrades FAA, indicating that compact residual selection and task-wise modulation are required to preserve task-specific behavior after temporal surgery.
Together, they enable the consistent gains achieved by \method{}.

\begin{figure}[t]
\centering

% -------- First figure --------
\raisebox{-1.22em}{\begin{minipage}[t]{0.38\linewidth}
  \centering
  \includegraphics[width=0.825\linewidth]{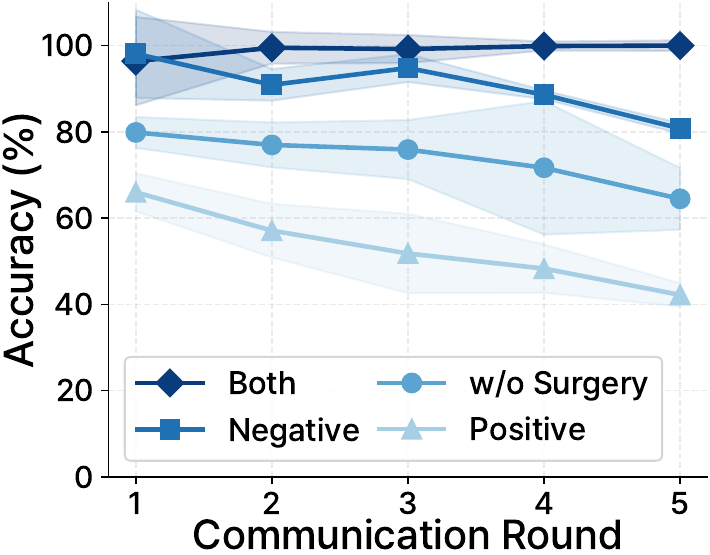}
  \captionof{figure}{\textbf{Ablation study on the target interaction for \method{}.}}
  \label{fig:surgery_ablation}
\end{minipage}}
\hfill
% -------- Second figure --------
\begin{minipage}[t]{0.60\linewidth}
  \centering
    \stepcounter{figure}
  \begin{subfigure}[t]{0.49\linewidth}
      \centering
      \includegraphics[width=\linewidth]{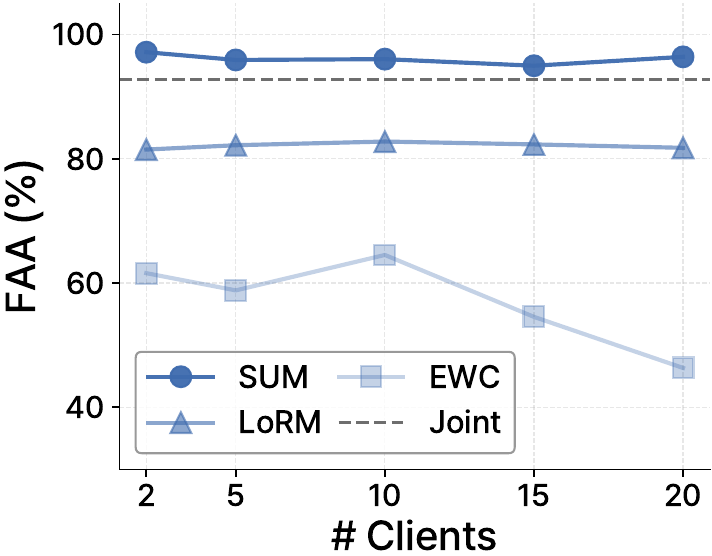}
      \caption{\textbf{Client scalability}}
      \label{fig:client_scale}
  \end{subfigure}
  \hfill
  \begin{subfigure}[t]{0.49\linewidth}
      \centering
      \includegraphics[width=\linewidth]{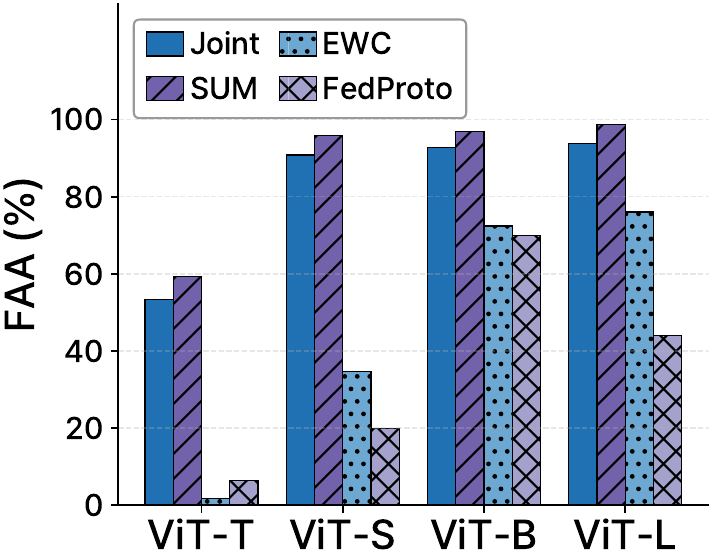}
      \caption{\textbf{Model scalability}}
      \label{fig:vit_scale}
  \end{subfigure}
    \addtocounter{figure}{-1}
  \captionof{figure}{\textbf{Scaling behavior of \method{}.}
  (a) Number of clients.
  (b) ViT backbone sizes.}
  \label{fig:scalability}

\end{minipage}

\end{figure}

\paragraph{Surgery Target.}
% By default, \method{} performs bidirectional surgery, handling both positive and negative cosine similarities of adaptation vectors to jointly remove redundant and conflicting components.
% A potential concern is that removing positively aligned components might discard beneficial shared information among clients or tasks.
% To verify this, we compare positive-only, negative-only, and bidirectional settings (\textit{both}).
% As shown in Fig.~\ref{fig:surgery_ablation}, the bidirectional variant consistently achieves the best performance, confirming that jointly resolving over-alignment and opposition between vectors yields the most stable aggregation.
\method{} handles both positive and negative cosine similarities of adaptation vectors to mitigate redundant or destructive updates.
A potential concern is that modifying positively aligned vectors may remove beneficial shared information.
To assess this, we compare three projection targets: positive-only, negative-only, and bidirectional (both) cosine interactions.
As shown in Fig.~\ref{fig:surgery_ablation} (CIFAR-100, $\beta=0.05$), the positive-only variant even underperforms the w/o-\method{} baseline, suggesting that naively altering positive directions disrupts the consistent shared representation. 
In contrast, \method{} operates on both positive and negative components while improving performance, preserving essential semantics and resolving both over-alignment and opposition.

\begin{figure}[t]
    \centering
    \includegraphics[width=0.65\linewidth]{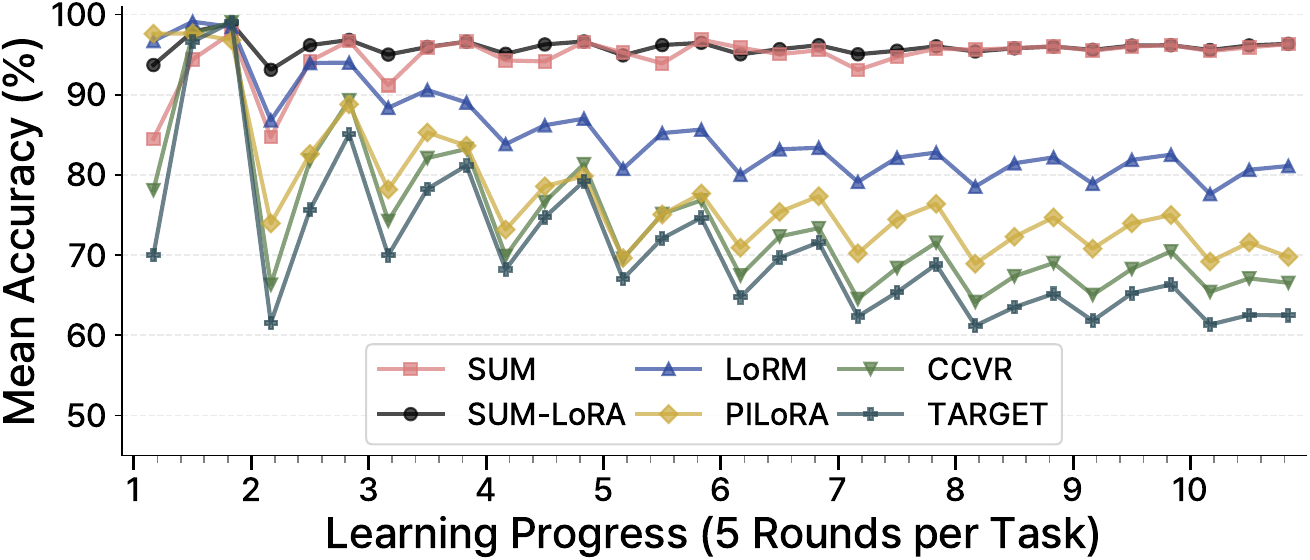}
\caption{\textbf{Learning curves of \method{}.}
A shallower V-shape and stable performance indicate fast adaptation and reduced forgetting.}
    \label{fig:learning_curves}
\end{figure}

\subsection{In-depth Analysis}
\label{subsec:analysis}

\paragraph{Scaling.}
We evaluate the client and model scalability of \method{} on CIFAR-100 ($\beta=0.05$) 
(Fig.~\ref{fig:client_scale} and Fig.~\ref{fig:vit_scale}). 
For client scaling, we compare against EWC~\cite{ewc}, a representative baseline sensitive to heterogeneity,
and LoRM~\cite{lorm}, a SOTA method. As the number of clients increases, EWC 
quickly degrades and LoRM plateaus below \textit{Joint}, while \method{} maintains and even 
surpasses \textit{Joint} performance. 
For model scaling, we include FedProto~\cite{fedproto}, the strongest full-finetuning approach (Tab.~\ref{tab:combined_eval}). 
As ViT size grows from Tiny to Large, FedProto and EWC show limited or negative gains, 
whereas \method{} scales smoothly and consistently outperforms \textit{Joint} baseline across all model sizes.

\paragraph{Learning Dynamics}
To analyze spatio-temporal learning dynamics, we compare baselines with \method{} in Fig.~\ref{fig:learning_curves} by sampling mean accuracies at the beginning, middle, and end of each task.
Across baselines, two consistent patterns emerge: (1) a \textit{V-shaped learning curve} during federated rounds and (2) a gradual performance decline as tasks progress.
The \textit{V-shape} reflects slower convergence in early rounds, while the decline indicates catastrophic forgetting.
In contrast, \method{} shows a much shallower V-shape and no noticeable drop across tasks, indicating faster adaptation and better retention.

\begin{figure}[t]
    \centering
    % ----- (a) Z-score threshold -----
    \begin{subfigure}[b]{0.325\linewidth}
        \centering
        \includegraphics[width=\linewidth]{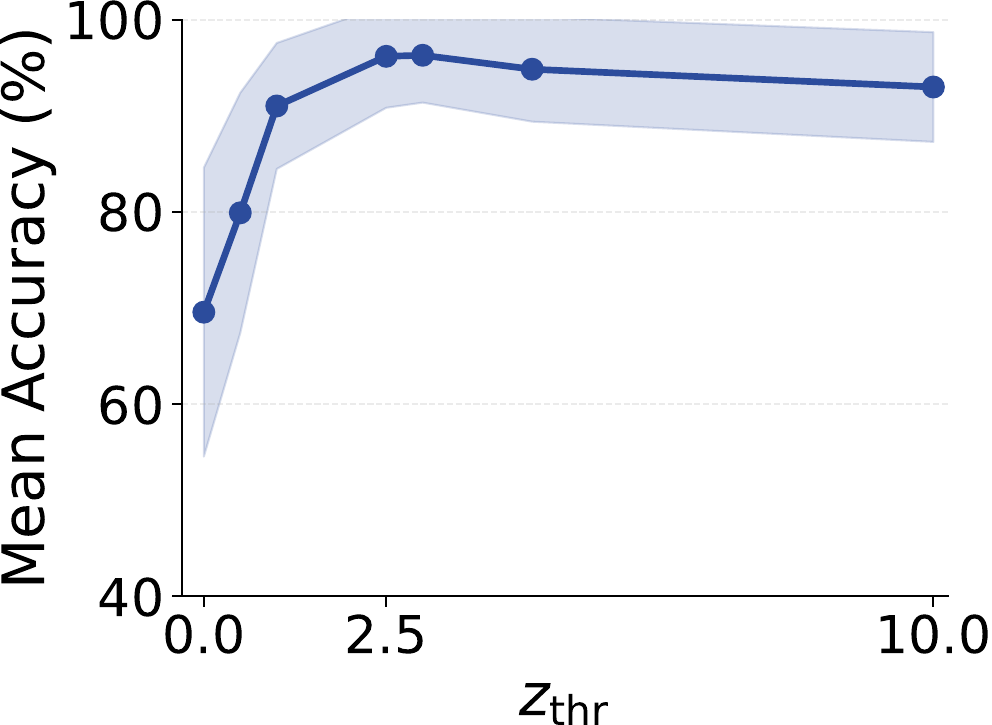}
        \caption{Z-score threshold }
        \label{fig:zthr}
    \end{subfigure}
    % ----- (b) Scaling coefficient -----
    \begin{subfigure}[b]{0.325\linewidth}
        \centering
        \includegraphics[width=\linewidth]{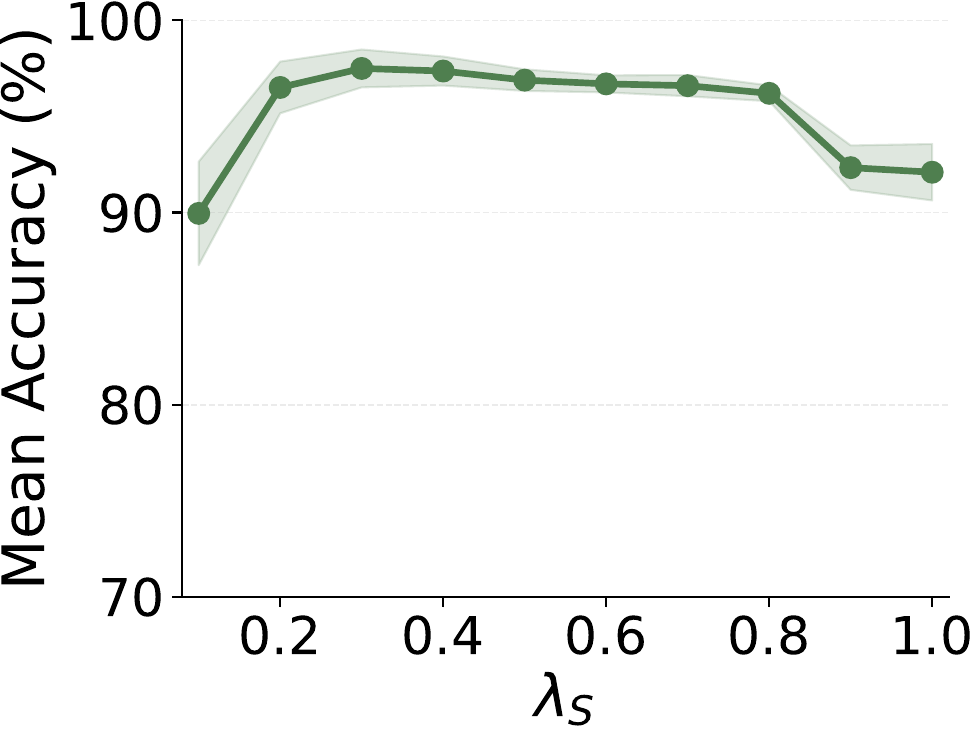}
        \caption{Scaling coefficient }
        \label{fig:lambda}
    \end{subfigure}
    % ----- (c) Top-k sparsification -----
    \begin{subfigure}[b]{0.325\linewidth}
        \centering
        \includegraphics[width=\linewidth]{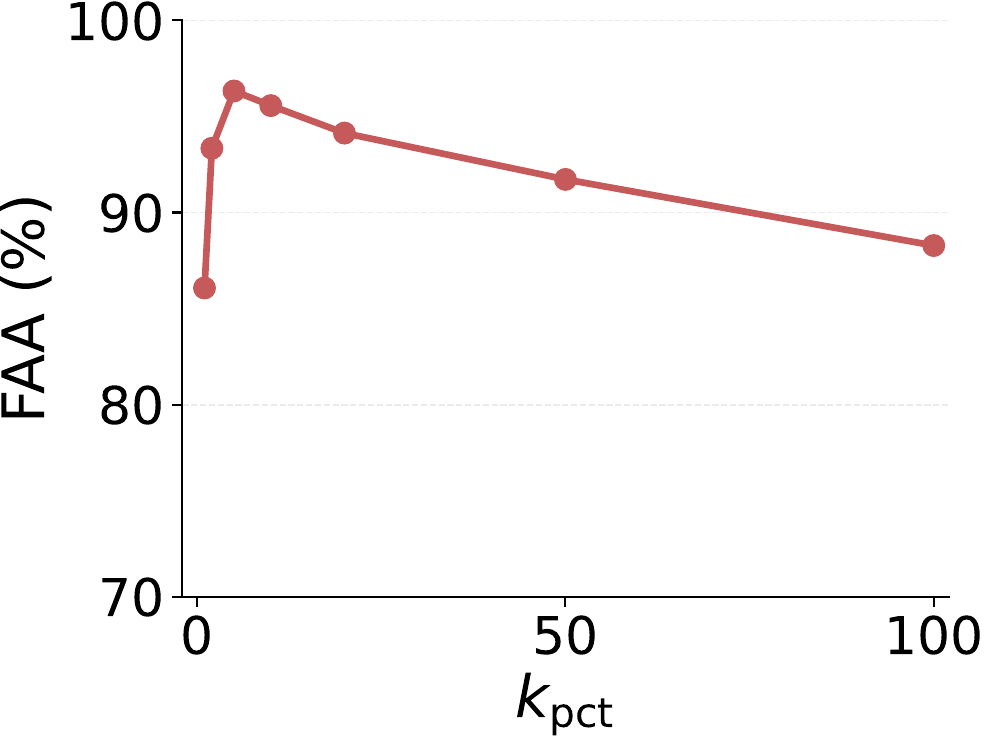}
        \caption{Top-$k$ sparsification }
        \label{fig:kpct}
    \end{subfigure}

    \caption{\textbf{Hyperparameter sensitivity analysis of \method{}.}}
    \label{fig:param_sensitivity}
\end{figure}

\paragraph{Hyperparameter Sensitivity}
We evaluate the robustness of \method{} to three hyperparameters: the Z-score threshold $z_{\text{thr}}$, 
the scaling coefficient $\lambda_S$, and the sparsification ratio $k_{\text{pct}}$. 
As shown in Fig.~\ref{fig:param_sensitivity}, \method{} remains stable across broad ranges, confirming its insensitivity to hyperparameter settings.
Accuracy peaks near $z_{\text{thr}}=3.0$, where large outliers are trimmed without removing informative parameters.
$\lambda_S$ performs consistently well between 0.2 and 0.8, with 0.4 as a balanced default.
Finally, moderate sparsification (up to 5\%) improves accuracy by suppressing residual interference and highlighting essential parameters, emphasizing the value of compact, interference-free representations in FCIL.

\paragraph{vs. Joint Training}
\textit{Joint} does not explicitly regulate directional interactions among gradients during centralized optimization. 
Prior work in MTL and CL shows that such interactions can lead to sharper minima and less separated representations~\cite{liu2021conflict,worsham2020multi,phan2025beyond}, implying that \textit{Joint} is not necessarily an oracle.
In contrast, \method{} explicitly removes redundant and destructive components from adaptation vectors before aggregation.
To examine this behavior, we analyze the local loss landscape and representation geometry on CIFAR-100 ($\beta=0.05$) test set, measuring sharpness as the loss increase under $\ell_2$ weight perturbations ($\rho=0.2$) and class separation via inter-class cosine distances between penultimate-layer representations.
As shown in Tab.~\ref{tab:sharpness_margin}, \method{} converges to flatter minima (sharpness $7.7\times10^{-4} \pm 1.2\times10^{-4}$ vs.\ $3.18\times10^{-3} \pm 4.5\times10^{-4}$) and larger margins ($0.482 \pm 0.021$ vs.\ $0.362 \pm 0.018$) than \textit{Joint}.
These results suggest that resolving directional conflicts before aggregation leads to more stable optimization and improved class separation, explaining why \method{} can outperform \textit{Joint}.
%better-separated representations.

\subsection{Efficiency and Practical Considerations}
\label{subsec:overhead}

\begin{table*}[t]
\centering
\caption{\textbf{Computational overhead comparison across FCIL methods.}
$D$ denotes the model parameter dimension, $T$ the number of tasks,
$r$ the LoRA rank, $D_{\text{attn}}$ the attention parameter dimension,
$P$ the prototype size, $G$ the Gram statistics size,
$B$ the orthogonal basis size, and $S$ the synthetic data size.}
\label{tab:overhead}
\setlength{\tabcolsep}{6pt}
\renewcommand{\arraystretch}{1.15}
\adjustbox{max width=0.95\linewidth}{
\begin{tabular}{lccccc}
\toprule
\textbf{Method} & \textbf{Upload} & \textbf{Download} & \textbf{Server Memory} & \textbf{Round Time (s)} & \textbf{Task Final. Time (s)} \\
\midrule
FedAvg  & $\mathcal{O}(D)$ & $\mathcal{O}(D)$ & $\mathcal{O}(D)$ & 239.4 & 9.1 \\
PiLoRA  & $\mathcal{O}(TrD_{\text{attn}} + P)$ & $\mathcal{O}(TrD_{\text{attn}} + TP)$ & $\mathcal{O}(D + TrD_{\text{attn}} + TP)$ & 269.3 & 0.6 \\
TARGET  & $\mathcal{O}(D)$ & $\mathcal{O}(D)$ & $\mathcal{O}(D + TS)$& 272.6 & 771.3 \\
FOT     & $\mathcal{O}(D)$ & $\mathcal{O}(D)$ & $\mathcal{O}(D + B)$ & 392.8 & 21.1 \\
LoRM    & $\mathcal{O}(D + G)$ & $\mathcal{O}(D + G)$ & $\mathcal{O}(D + G)$ & 689.7 & 24.5 \\
\midrule
\method{} & $\mathcal{O}(D)$ & $\mathcal{O}(D)$ & $\mathcal{O}(TD)$ & 324.2 & 37.7 \\
\bottomrule
\end{tabular}}
\end{table*}

\begin{figure}[t]
    \centering

    \begin{minipage}[t]{0.55\linewidth}
        \vspace{0pt}
        \centering
        \input{fig_tex/task_scaling}
        \vspace{-1mm}
        \caption{\footnotesize
        Memory footprint and mitigation for server-side task scalability of \method{}.
        }
        \label{fig:task_scaling}
    \end{minipage}
    \hfill
    \begin{minipage}[t]{0.42\linewidth}
        \vspace{0pt}
        \centering
        \captionof{table}{\footnotesize
        Robustness to unreliable clients on CIFAR-100 ($\beta=0.05$).
        }
        \vspace{3mm}
        {\footnotesize
        {\footnotesize
\setlength{\tabcolsep}{2.0pt}
\renewcommand{\arraystretch}{1.0}

\adjustbox{max width=0.9\linewidth}{%
\begin{tabular}{lcc}
\toprule
\textbf{Method} & \textbf{Clean} & \textbf{Avg. FAA} $\uparrow$ \\
\midrule
LoRM & 82.8 & 67.7 $\pm$ 8.7 (\textcolor{red}{$\downarrow$18.2\%}) \\
\textbf{\method{}} & \textbf{96.3} & \textbf{80.5 $\pm$ 9.6} (\textcolor{red}{$\downarrow$16.4\%}) \\
\bottomrule
\end{tabular}%
}
}
        }
        % \vspace{-1mm}
        % \captionof{table}{\footnotesize
        % Robustness to unreliable clients on CIFAR-100 ($\beta=0.05$).
        % }
        \label{tab:unreliable_clients}
    \end{minipage}

    \vspace{-2mm}
\end{figure}

We analyze the computational overhead of \method{} compared to representative FCIL baselines.
Tab.~\ref{tab:overhead} summarizes communication complexity, server memory requirements, and measured runtime on a single H200 GPU, including round time and task finalization time.
Although \method{} operates on adaptation vectors, it introduces only lightweight server-side computation while keeping communication costs identical to standard FL updates.

\paragraph{Memory Footprint and Mitigation.}
In the worst case, maintaining past refined task vectors on the server scales as $O(TD)$, where $T$ is the number of tasks and $D$ is the model parameter dimension.
Since this cost is entirely server-side, it does not introduce additional client memory, computation, or communication, but it can become non-negligible as the number of tasks grows.
We therefore evaluate simple task-vector compression strategies: low-precision storage using FP16 or INT8, and top-$k$ sparsification that keeps only the largest coordinates of each task vector.
As shown in Fig.~\ref{fig:task_scaling}, these compression schemes substantially reduce temporal storage up to 50 tasks while preserving the qualitative FAA trend.
This suggests that the $O(TD)$ storage cost is manageable in practice through lightweight server-side compression.

\paragraph{Unreliable Client Robustness.}
To evaluate robustness under unreliable participants, we simulate malicious clients by corrupting local training labels.
Specifically, 1--4 out of 10 clients are trained with shifted ground-truth labels to produce misleading updates.
As shown in Tab.~\ref{tab:unreliable_clients}, \method{} maintains substantially higher accuracy and exhibits smaller performance degradation than LoRM~\cite{lorm}.
This behavior arises because \method{} refines adaptation vectors at aggregation time, suppressing incompatible directions introduced by corrupted client updates.
\section{Conclusion}
\label{sec:conclusion}

We introduced \methodlong{} (\method{}), a unified framework for Federated Class Incremental Learning (FCIL) that models federated and continual updates as spatio-temporal interactions among adaptation vectors. 
By performing projection-based surgery at aggregation time, \method{} removes redundant and conflicting update components, enabling stable integration of heterogeneous client and task updates.
Extensive experiments show that \method{} consistently improves performance across FCIL benchmarks while maintaining lightweight communication and computation overhead. 
Importantly, this improvement is achieved without modifying client-side training, adding replay buffers, or requiring additional client communication, making \method{} compatible with standard federated deployment protocols.
These results highlight the importance of modeling directional interactions among adaptation vectors for resolving spatio-temporal interference in FCIL.

\bibliographystyle{splncs04}
\bibliography{main}

\clearpage
\appendix
\setcounter{page}{1}

\section*{Appendix}
\addcontentsline{toc}{section}{Appendix}

In this supplementary material, we provide additional theoretical analyses, implementation details, and extended experimental results that support the main findings of the paper. Specifically:

\begin{itemize}
    \item     \textbf{Section~\ref{sup:discussion}} discusses the limitations of the proposed \method{} and potential strategies for mitigating them. 
    \item
\textbf{Section~\ref{sup:proof}} presents the formal proofs of the main theorems and propositions introduced in the main text.
    \item \textbf{Section~\ref{sup:implementation_details}} describes the task setup, hyperparameters, and other implementation details necessary for reproduction.
    \item \textbf{Section~\ref{sup:additional_experiments}} provides additional ablation studies, extended quantitative evaluations.
\end{itemize}
\section{Limitations and Mitigation Strategies}
\label{sup:discussion}

\subsection{Server-Side Storage Overhead.}
Temporal \method{} stores refined task adaptation vectors across tasks, leading to a worst-case server storage requirement of $\mathcal{O}(TD)$, where $T$ denotes the number of tasks and $D$ the number of model parameters. 
This storage corresponds to persistent model parameters maintained on the server rather than runtime memory, and therefore does not affect client computation or communication. 
However, it may become costly when the number of tasks grows very large.

\paragraph{Mitigation.}
As described in Sec.~\ref{subsec:overhead}, the storage cost can be reduced by compressing stored task adaptation vectors. 
For example, each vector may retain only the top-$k$ coordinates with the largest magnitudes or be represented using a low-rank factorization.
These compressed representations significantly reduce storage while preserving the dominant task adaptation directions.

\subsection{Aggregation Overhead.}
As the number of adaptation vectors increases, the refinement step in both spatial and temporal \method{} requires projections against a growing set of vectors. 
In the spatial case, these vectors correspond to client updates, while in the temporal case they correspond to stored task adaptation vectors.
This may increase the computational cost of the server-side aggregation step when the number of clients or tasks becomes very large.

\paragraph{Mitigation.}
In practice, the projection step can be restricted to a limited subset of adaptation vectors rather than the entire set.
For example, projections may be performed only against a subset of stored vectors, such as a small set of representative directions.
Such strategies keep the number of projection operations bounded while preserving most of the benefits of interference-aware refinement.

\subsection{Inference-Time Task Selection}
The inference-ready module assumes that the relevant task adaptation vectors can be selected at inference time.
In practice, this requires identifying which task (or subset of tasks) is most relevant for a given input.
This assumption is common in MTL and modular model compositions~\cite{huang2024emr,yadav2024ties,yang2023adamerging}, where task-specific components are activated depending on the evaluation setting.

\paragraph{Mitigation.}
In practice, task selection can be implemented using a lightweight retrieval mechanism.
For example, the server may maintain a small key–value pool~\cite{l2p} over adaptation vectors and retrieve the most relevant vectors given the input query.
Such a module can be implemented with minimal overhead and does not affect the core aggregation procedure.

\section{Theoretical Proofs and Extensions}
\label{sup:proof}
Throughout this section, we assume that each local objective $\mathcal{L}_i$ is differentiable with $L_i$-Lipschitz continuous gradients, that the local iterates $\{\theta_i^{t,s}\}_{s=0}^{K_i}$ remain in a bounded neighborhood of $\theta_G^t$ within each communication round $t$, and that the effective step size $\eta_{\mathrm{eff},i} = \sum_{s=1}^{K_i} \eta_i^s$ is sufficiently small so that $\mathcal{O}(\eta_{\mathrm{eff},i}^{\,2})$ terms are negligible under the usual first-order approximation.

\subsection{FCIL as Multi-Task Learning.} 
\label{sup:fcil_mtl}

Federated Learning (FL) optimizes a shared global model by aggregating local objectives across participating clients.  
Let $\mathcal{S}$ denote the set of clients and $\mathcal{D}_i$ the local dataset of client $i$.  
The objective is:
\begin{equation}
\min_{\theta} \mathcal{L}_{\text{Spatial}}(\theta)
= \min_{\theta} \sum_{i \in \mathcal{S}} \mathcal{L}_i(\mathcal{D}_i; \theta),
\end{equation}
where each client can be interpreted as a distinct objective under shared parameters, forming a decentralized instance of \textit{spatial MTL}.  
Continual Learning (CL) similarly optimizes a shared model across a sequence of tasks with datasets $\{\mathcal{D}_k\}_{k=1}^T$:
\begin{equation}
\min_{\theta} \mathcal{L}_{\text{Temporal}}(\theta)
= \min_{\theta} \sum_{k=1}^{T} \mathcal{L}_k(\mathcal{D}_k; \theta),
\end{equation}
defining a \textit{temporal MTL} problem in which objectives evolve over time under the same model.  
Taken together, FCIL can be viewed as a two-axis extension of MTL, where spatial and temporal objectives jointly interact within a shared parameter space.  
This unified formulation provides a natural foundation for analyzing and mitigating interference across both dimensions.  
Empirically, we also observe that the parameter updates induced by FL and CL exhibit near-orthogonal growth patterns, consistent with well-known behaviors in MTL~\cite{ilharco2022editing} (Appendix~\ref{supsec:task_relation_analysis}).

\subsection{Adaptation Vector as Gradients}
\label{sup:proof_prop1}

\begin{proposition}
If each $\mathcal{L}_i$ has $L_i$-Lipschitz continuous gradients and local updates remain in a small neighborhood of $\theta_G^t$, then
\[
\tau_i^t \;\approx\; -\,\eta_i^{\mathrm{eff}}\,\nabla \mathcal{L}_i(\theta_G^t),
\qquad \text{where }\;\eta_i^{\mathrm{eff}}=\sum_{s=1}^{K_i}\eta_i^s,
\]
with an approximation error scaling with $(\eta_i^{\mathrm{eff}})^2$ up to constants depending on $L_i$.
\label{prop:client_vector_gradient}
\end{proposition}
\textit{Proof.}
Fix a round $t$ and client $i$, and denote the global model by $\theta_G^t$.
We write the local gradient descent trajectory as
\begin{align}
  \theta_{i}^{t,0} &= \theta_G^t, \\
  \theta_{i}^{t,s} &= \theta_{i}^{t,s-1} 
  - \eta_i^s \,\nabla \mathcal{L}_i\bigl(\theta_{i}^{t,s-1}\bigr),
  \quad s = 1,\dots,K_i,
\end{align}
where $\eta_i^s$ is the local step size at step $s$. The resulting client adaptation vector is
\begin{align}
  \tau_i^t 
  := \theta_{i}^{t,K_i} - \theta_G^t
  = - \sum_{s=1}^{K_i} \eta_i^s \,
      \nabla \mathcal{L}_i\bigl(\theta_{i}^{t,s-1}\bigr).
\end{align}
Let $\eta_{\mathrm{eff},i} := \sum_{s=1}^{K_i} \eta_i^s$.
Adding and subtracting $\nabla \mathcal{L}_i(\theta_G^t)$ inside the sum gives
\begin{align}
  \tau_i^t
  &= - \sum_{s=1}^{K_i} \eta_i^s 
     \Bigl[
       \nabla \mathcal{L}_i(\theta_G^t)
       + \bigl(
          \nabla \mathcal{L}_i(\theta_{i}^{t,s-1})
          - \nabla \mathcal{L}_i(\theta_G^t)
         \bigr)
     \Bigr] \nonumber \\
  &= - \eta_{\mathrm{eff},i} \,
      \nabla \mathcal{L}_i(\theta_G^t)
     - \sum_{s=1}^{K_i} \eta_i^s \,
       \bigl(
         \nabla \mathcal{L}_i(\theta_{i}^{t,s-1})
         - \nabla \mathcal{L}_i(\theta_G^t)
       \bigr).
\end{align}
Define the approximation error
\begin{align}
  E_i^t 
  &:= \tau_i^t + \eta_{\mathrm{eff},i}\,\nabla \mathcal{L}_i(\theta_G^t) \nonumber\\
  &= - \sum_{s=1}^{K_i} \eta_i^s 
      \left(
        \nabla \mathcal{L}_i(\theta_{i}^{t,s-1})
        - \nabla \mathcal{L}_i(\theta_G^t)
      \right).
\end{align}
Since $\nabla \mathcal{L}_i$ is $L_i$-Lipschitz,
\[
\|\nabla \mathcal{L}_i(\theta) - \nabla \mathcal{L}_i(\theta')\|
 \le L_i \|\theta - \theta'\|,
\]
we obtain
\begin{align}
  \|E_i^t\|
  &\le \sum_{s=1}^{K_i} \eta_i^s \,
       \bigl\|
         \nabla \mathcal{L}_i(\theta_{i}^{t,s-1})
         - \nabla \mathcal{L}_i(\theta_G^t)
       \bigr\| \nonumber \\
  &\le L_i \sum_{s=1}^{K_i} \eta_i^s \,
       \bigl\|\theta_{i}^{t,s-1} - \theta_G^t\bigr\|.
  \label{eq:error-bound-1}
\end{align}
From the update rule,
\begin{align}
  \theta_{i}^{t,s-1} - \theta_G^t
  &= - \sum_{r=1}^{s-1} \eta_i^r \,
      \nabla \mathcal{L}_i\bigl(\theta_{i}^{t,r-1}\bigr),
\end{align}
and therefore by the triangle inequality
\begin{align}
  \bigl\|\theta_{i}^{t,s-1} - \theta_G^t\bigr\|
  \le \sum_{r=1}^{s-1} \eta_i^r \,
      \bigl\|\nabla \mathcal{L}_i(\theta_{i}^{t,r-1})\bigr\|.
\end{align}
By assumption, the local iterates stay in a small neighborhood of 
$\theta_G^t$, hence the gradient norm is bounded there.
That is, there exists $G_i > 0$ such that
\[
\|\nabla \mathcal{L}_i(\theta)\| \le G_i
\]
along the trajectory.
Thus
\begin{align}
  \bigl\|\theta_{i}^{t,s-1} - \theta_G^t\bigr\|
  \le G_i \sum_{r=1}^{s-1} \eta_i^r.
\end{align}
Substituting into~\eqref{eq:error-bound-1} yields
\begin{align}
  \|E_i^t\|
  &\le L_i G_i
      \sum_{s=1}^{K_i} \eta_i^s
      \sum_{r=1}^{s-1} \eta_i^r.
\end{align}
The double sum can be rewritten as
\begin{align}
  \sum_{s=1}^{K_i} \sum_{r=1}^{s-1} \eta_i^s \eta_i^r
  &= \frac{1}{2}
     \Biggl[
       \Bigl(\sum_{s=1}^{K_i} \eta_i^s\Bigr)^{\!2}
       - \sum_{s=1}^{K_i} (\eta_i^s)^2
     \Biggr]
  \le \frac{1}{2} \,\eta_{\mathrm{eff},i}^{\,2}.
\end{align}
Therefore
\begin{align}
  \|E_i^t\|
  \le \frac{1}{2} L_i G_i \,\eta_{\mathrm{eff},i}^{\,2}.
\end{align}
Hence we obtain
\begin{align}
  \tau_i^t
  = - \eta_{\mathrm{eff},i} \,\nabla \mathcal{L}_i(\theta_G^t)
    + \varepsilon_i^t,
  \quad
  \|\varepsilon_i^t\|
  \le C_i \,\eta_{\mathrm{eff},i}^{\,2},
\end{align}
for a constant $C_i := \tfrac{1}{2} L_i G_i$ depending only on $L_i$ and the local neighborhood of $\theta_G^t$.
Equivalently,
\[
\tau_i^t \approx - \eta_{\mathrm{eff},i} \nabla \mathcal{L}_i(\theta_G^t)
\]
with approximation error scaling as $\mathcal{O}(\eta_{\mathrm{eff},i}^{\,2})$. Task adaptation vectors are obtained from the aggregation of client updates across communication rounds, and therefore inherit the same first-order gradient interpretation established in Prop.~\ref{prop:client_vector_gradient}.
\hfill$\square$

\subsection{Extension of Proposition~\ref{prop:client_vector_gradient} to Multi-step SGD}
\label{sup:sgd_extension}

The result in Prop.~\ref{prop:client_vector_gradient} extends naturally
to the stochastic setting.  
Let the local update rule be
\begin{align}
  \theta_{i}^{t,s}
  &= \theta_{i}^{t,s-1}
     - \eta_i^s\, g_i^{t,s-1}, 
\end{align}
where $g_i^{t,s-1}$ is a stochastic gradient estimate satisfying
\[
\mathbb{E}[g_i^{t,s-1}\mid\theta_{i}^{t,s-1}]
 = \nabla \mathcal{L}_i(\theta_{i}^{t,s-1})
\]
and
\[
\mathbb{E}\|g_i^{t,s-1}-\nabla \mathcal{L}_i(\theta_{i}^{t,s-1})\|^2
 \le \sigma_i^2
\]
for some bounded variance $\sigma_i^2$.
The client adaptation vector becomes
\begin{align}
  \tau_i^t
  &= - \sum_{s=1}^{K_i} \eta_i^s\, g_i^{t,s-1}.
\end{align}
Adding and subtracting $\nabla \mathcal{L}_i(\theta_G^t)$ yields
\begin{align}
  \tau_i^t
  &= - \eta_{\mathrm{eff},i}\,\nabla \mathcal{L}_i(\theta_G^t)
     \nonumber\\[-0.4em]
  &\quad - \sum_{s=1}^{K_i} \eta_i^s
     \bigl(
       \nabla \mathcal{L}_i(\theta_{i}^{t,s-1})
       - \nabla \mathcal{L}_i(\theta_G^t)
     \bigr)
     \nonumber\\[-0.4em]
  &\quad - \sum_{s=1}^{K_i} \eta_i^s
     \bigl(
       g_i^{t,s-1}
       - \nabla \mathcal{L}_i(\theta_{i}^{t,s-1})
     \bigr),
\end{align}
where $\eta_{\mathrm{eff},i}=\sum_{s=1}^{K_i}\eta_i^s$.
The first two terms reproduce the deterministic GD decomposition.
Define
\begin{align}
  E_i^t
  &:= \tau_i^t + \eta_{\mathrm{eff},i}\,
       \nabla \mathcal{L}_i(\theta_G^t).
\end{align}
Thus $E_i^t$ contains both the Lipschitz drift term and an additional stochastic noise term.
Under the same assumptions as in Prop.~\ref{prop:client_vector_gradient},
the deterministic drift term continues to scale as
\[
\mathcal{O}(\eta_{\mathrm{eff},i}^{\,2}).
\]
For the stochastic term, standard variance accumulation yields
\begin{align}
  \mathbb{E}\Bigl\|
     \sum_{s=1}^{K_i}
        \eta_i^s
        \bigl(
          g_i^{t,s-1}
          - \nabla \mathcal{L}_i(\theta_{i}^{t,s-1})
        \bigr)
  \Bigr\|^2
  \le
  \sigma_i^2
  \sum_{s=1}^{K_i} (\eta_i^s)^2.
\end{align}
Since
\[
\sum_{s=1}^{K_i} (\eta_i^s)^2
 \le \eta_{\mathrm{eff},i}^{\,2},
\]
the stochastic deviation is bounded in expectation by
\[
\mathcal{O}(\eta_{\mathrm{eff},i})
\]
in norm (or $\mathcal{O}(\eta_{\mathrm{eff},i}^{2})$ in mean-square).
Therefore
\begin{align}
  \mathbb{E}[\tau_i^t]
  &= - \eta_{\mathrm{eff},i}
     \nabla \mathcal{L}_i(\theta_G^t)
     + \mathcal{O}(\eta_{\mathrm{eff},i}^{\,2}),
\end{align}
and the deviation satisfies
\begin{align}
  \tau_i^t
  &= - \eta_{\mathrm{eff},i}
     \nabla \mathcal{L}_i(\theta_G^t)
     + \varepsilon_i^t
     + \xi_i^t,
\end{align}
where
\[
\|\varepsilon_i^t\| = \mathcal{O}(\eta_{\mathrm{eff},i}^{\,2})
\]
is the deterministic Lipschitz drift term and
\[
\mathbb{E}\|\xi_i^t\|^2
 = \mathcal{O}(\eta_{\mathrm{eff},i}^{\,2})
\]
captures the accumulated stochastic noise.
In particular, the \emph{direction} of the client adaptation vector remains aligned with
\[
-\nabla \mathcal{L}_i(\theta_G^t)
\]
up to second-order error, and its expectation preserves the same first-order equivalence as in the deterministic case.
\hfill$\square$

\subsection{Theorem~\ref{thm:main}}
\label{app:proof_main}

\textit{Proof.}
We first recall the standard two-task MTL setting in Gradient Surgery (PCGrad)~\cite{yu2020gradient}.
Let $\mathcal{L}_1, \mathcal{L}_2$ be two task losses and
$\mathcal{L}(\theta) = \mathcal{L}_1(\theta) + \mathcal{L}_2(\theta)$.
At a parameter $\theta$, denote
\[
\mathbf g_1 := \nabla \mathcal L_1(\theta),\quad
\mathbf g_2 := \nabla \mathcal L_2(\theta),\quad
\mathbf g := \mathbf g_1 + \mathbf g_2 .
\]
The standard MTL update and the PCGrad update with step size $t$ are
\begin{align}
\theta^{\text{MT}}
&= \theta - t\,\mathbf g,\\
\theta^{\text{PCGrad}}
&= \theta - t\,\mathbf g^{\text{pc}},
\end{align}
where the symmetric two-task PCGrad update removes projection components
between $\mathbf g_1$ and $\mathbf g_2$:
\begin{align}
\mathbf g^{\text{pc}}
&=
\mathbf g
-
\frac{\mathbf g_1^\top \mathbf g_2}{\|\mathbf g_1\|_2^2}\,\mathbf g_1
-
\frac{\mathbf g_1^\top \mathbf g_2}{\|\mathbf g_2\|_2^2}\,\mathbf g_2 .
\end{align}
Under differentiability, $L$-Lipschitz continuity of $\nabla\mathcal L$,
and a lower bound on the multi-task curvature
$\mathcal H(\mathcal L; \theta,\theta^{\text{MT}})$,
Yu et al.~\cite{yu2020gradient} show that suitable constants
$\ell\le L$ and step size $t$ exist such that
\begin{align}
\mathcal L(\theta^{\text{PCGrad}})
\le
\mathcal L(\theta^{\text{MT}}).
\label{eq:pcgrad-guarantee}
\end{align}
We now interpret this result in the model-merging setting via task vectors.
Let $\theta_{\text{base}}$ be the common base model and
$\theta_1,\theta_2$ the fine-tuned models for the two tasks,
with task vectors
\begin{align}
\tau_1 := \theta_1 - \theta_{\text{base}},\qquad
\tau_2 := \theta_2 - \theta_{\text{base}}.
\end{align}
By Prop.~\ref{prop:client_vector_gradient} and its stochastic extension
(Appendix~\ref{sup:sgd_extension}),
there exist effective step sizes $t_1^*,t_2^*>0$ such that
\begin{align}
\tau_1 &\approx -t_1^*\,\mathbf g_1(\theta_{\text{base}}),\\
\tau_2 &\approx -t_2^*\,\mathbf g_2(\theta_{\text{base}}),
\end{align}
up to second-order error.
Consider first the standard model merging
(\eg FedAvg~\cite{fedavg}, Task Arithmetic~\cite{ilharco2022editing},
MagMax~\cite{magmax}):
\begin{align}
\theta_{\text{standard}}
=
\theta_{\text{base}}
+
\lambda_1\tau_1
+
\lambda_2\tau_2 .
\end{align}
Choosing a scalar $t^*>0$ and setting
$\lambda_j=t^*/t_j^*$ gives
\begin{align}
\theta_{\text{standard}}
&=
\theta_{\text{base}}
+
\frac{t^*}{t_1^*}\tau_1
+
\frac{t^*}{t_2^*}\tau_2 \\
&\approx
\theta_{\text{base}}
-
t^*
\bigl(
\mathbf g_1(\theta_{\text{base}})
+
\mathbf g_2(\theta_{\text{base}})
\bigr),
\label{eq:std-as-mt}
\end{align}
which corresponds to a single MTL gradient step.
Next we consider our surgery-based merging \method{},
which refines task vectors via projection:
\begin{align}
\hat{\tau}_1
&=
\tau_1
-
\frac{\tau_1^\top\tau_2}{\|\tau_2\|_2^2}\tau_2,\\
\hat{\tau}_2
&=
\tau_2
-
\frac{\tau_1^\top\tau_2}{\|\tau_1\|_2^2}\tau_1 .
\end{align}
The merged model becomes
\begin{align}
\theta_{\method{}}
=
\theta_{\text{base}}
+
\lambda_1\hat{\tau}_1
+
\lambda_2\hat{\tau}_2 .
\end{align}
Substituting the gradient expressions yields
\begin{align}
\theta_{\method{}}
\approx
\theta_{\text{base}}
-
t^*\mathbf g^{\text{pc}},
\label{eq:method-as-pc}
\end{align}
which matches a single PCGrad update at $\theta=\theta_{\text{base}}$.
Thus
\begin{align}
\theta_{\text{standard}}
&\approx \theta^{\text{MT}},\\
\theta_{\method{}}
&\approx \theta^{\text{PCGrad}}.
\end{align}
Applying the PCGrad guarantee in
Eq.~\eqref{eq:pcgrad-guarantee} therefore yields
\begin{align}
\mathcal L(\theta_{\method{}})
\le
\mathcal L(\theta_{\text{standard}}),
\end{align}
which proves the theorem.
\hfill$\square$

\subsection{Temporal Extension of Theorem~\ref{thm:main}}
\label{app:proof_main_temporal}

\textit{Proof.}
Let the sequential global updates be
\[
\Delta_k := \theta_G^{(k)} - \theta_G^{(k-1)},
\]
and define the accumulated task vectors
\[
\tau_k := \theta_G^{(k)} - \theta_{\text{pre}}
= \sum_{i=1}^{k} \Delta_i .
\]
Thus all $\tau_k$ share the same base model $\theta_{\text{pre}}$,
satisfying the common-base requirement of Theorem~\ref{thm:main}.
Online Surgery performs the following sequential refinement
\begin{align}
\hat{\tau}_1 &= \tau_1, \\
\hat{\tau}_k
&=
\tau_k
-
\sum_{j=1}^{k-1}
\frac{\tau_k^\top \hat{\tau}_j}{\|\hat{\tau}_j\|_2^2}\hat{\tau}_j .
\label{eq:online_surgery_appendix}
\end{align}
At each step $k$, we consider the two vectors
\[
\sum_{j<k}\hat{\tau}_j
\quad\text{and}\quad
\tau_k .
\]
These vectors share the same base $\theta_{\text{pre}}$, and therefore form
exactly the two-task setting required by Theorem~\ref{thm:main}.
Applying one step of surgery yields
\[
\mathcal{L}\!\left(
\theta_{\text{pre}}
+
\sum_{j<k}\hat{\tau}_j
+
\hat{\tau}_k
\right)
\le
\mathcal{L}\!\left(
\theta_{\text{pre}}
+
\sum_{j<k}\hat{\tau}_j
+
\tau_k
\right).
\]
The base case $k=1$ is trivial since $\hat{\tau}_1=\tau_1$.
Assuming the inequality holds up to step $k-1$,
substituting into the above inequality yields the same guarantee for step $k$.
Induction over $k=1,\dots,T$ therefore gives
\[
\mathcal{L}\!\left(
\theta_{\text{pre}}
+
\sum_{k=1}^{T}\hat{\tau}_k
\right)
\le
\mathcal{L}\!\left(
\theta_{\text{pre}}
+
\sum_{k=1}^{T}\tau_k
\right).
\]
Unlike the original PCGrad update, which removes projections only when
the pairwise inner product is negative, our surgery operator subtracts the
projection component for any nonzero directional coupling.
Accordingly, this result should be viewed as a local descent-bound
interpretation of the same projection principle under the stated smoothness
and common-base assumptions, rather than as an unconditional improvement
guarantee.
In particular, the conclusion remains conditional on the
signal--interference condition stated in Thm.~\ref{thm:main}.

\hfill$\square$

\subsection{Proposition~\ref{prop:temporal_equivalence}}
\label{app:proof_prop_temporal}

\textit{Proof.}
We show that applying the projection in Eq.~\ref{eq:general_projection}
to the accumulated task vectors
\[
\tau_k = \theta_G^{(k)} - \theta_{\text{pre}}
\]
is equivalent to applying it to the sequential increments
\[
\Delta_k = \theta_G^{(k)} - \theta_G^{(k-1)} .
\]
Because $\theta_G^{(0)} = \theta_{\text{pre}}$, we have
\[
\tau_k
=
\theta_G^{(k)}-\theta_{\text{pre}}
=
\sum_{i=1}^{k}
(\theta_G^{(i)}-\theta_G^{(i-1)})
=
\sum_{i=1}^{k}\Delta_i .
\]
Thus
\[
\tau_1=\Delta_1,
\qquad
\tau_2=\Delta_1+\Delta_2 .
\]
Consider the case $T=2$.
Applying the projection operator gives
\begin{align}
\hat{\tau}_2
&=
\tau_2-\operatorname{proj}(\tau_2,\tau_1) \\
&=
(\Delta_1+\Delta_2)-\operatorname{proj}(\Delta_1+\Delta_2,\Delta_1) \\
&=
(\Delta_1+\Delta_2)
-
\bigl(\Delta_1+\operatorname{proj}(\Delta_2,\Delta_1)\bigr) \\
&=
\Delta_2-\operatorname{proj}(\Delta_2,\Delta_1)
=
\hat{\Delta}_2 .
\end{align}
Hence $\hat{\tau}_2=\hat{\Delta}_2$.
For $k>2$, since $\tau_k=\sum_{i=1}^{k}\Delta_i$, the same algebraic
cancellation applies at each step.
Assuming $\hat{\tau}_j=\hat{\Delta}_j$ for all $j<k$,
the projection of $\tau_k$ onto the span of
$\{\hat{\tau}_1,\dots,\hat{\tau}_{k-1}\}$ reduces to the projection of
$\Delta_k$ onto the span of
$\{\hat{\Delta}_1,\dots,\hat{\Delta}_{k-1}\}$,
which yields $\hat{\tau}_k=\hat{\Delta}_k$.
Thus projection on accumulated vectors $\{\tau_k\}$ produces exactly the
same refined directions as projection on the sequential increments
$\{\Delta_k\}$.
\hfill$\square$
\section{Implementation Details}
\label{sup:implementation_details}

\subsection{Hyperparameters}
\label{supsec:hyperparameters}

\begin{table}[t]
\centering
\footnotesize
\caption{\textbf{List of hyperparameter acronyms used in the paper.}}
\renewcommand{\arraystretch}{1.0}
\setlength{\tabcolsep}{0.4em}
\adjustbox{max width=0.65\textwidth}{
\begin{tabular}{cc}
\toprule
\textbf{Acronym} & \textbf{Description} \\
\midrule
\textit{lr} & Learning rate \\
\textit{lr\textsubscript{pr}} & Prototype learning rate \\
$\lambda_{\text{KL}}$ & KD loss scaling factor \\
\textit{r} & Rank for low-rank decomposition \\
$\textit{g}_{ep}$ & Generator training/generation epochs \\
$\epsilon_{\text{fot}}$, $\epsilon_{\text{inc}}$ & FOT energy threshold and per-task epsilon increment \\
$\gamma$ & Gram-matrix off-diagonal decay \\
\midrule
$z_{\text{thr}}$ & Z-score trimming threshold (Spatial \method{}) \\
$\lambda_S$ & Client-vector scaling (Spatial \method{}) \\
$k_{\text{pct}}$ & Top-$k$ sparsification percent (Inference-ready Module) \\
\bottomrule
\end{tabular}}
\label{tab:acronyms}
\end{table}

% \begin{table}[h]
% \centering
% \footnotesize
% \caption{\textbf{Overview of hyperparameter tables for each dataset.}}
% \renewcommand{\arraystretch}{1.0}
% \setlength{\tabcolsep}{1.5em}
% \adjustbox{max width=0.5\textwidth}{
% \begin{tabular}{lll}
% \toprule
% \textbf{Domain} & \textbf{Dataset} & \textbf{Table Ref.} \\
% \midrule
% \multirow{6}{*}{Vision} 
%   & CIFAR-100    & Table~\ref{tab:cifar_hyperparam} \\
%   & ImageNet-R   & Table~\ref{tab:imagenetr_hyperparam} \\
%   & ImageNet-A   & Table~\ref{tab:imageneta_hyperparam} \\
%   & EuroSAT      & Table~\ref{tab:eurosat_hyperparam} \\
%   & Cars-196     & Table~\ref{tab:cars_hyperparam} \\
%   & CUB-200      & Table~\ref{tab:cub_hyperparam} \\
% \midrule
% \multirow{2}{*}{Language}
% & 20-Newsgroups  & Table~\ref{tab:20ng_hyperparam} \\
%   & CLINC-150      & Table~\ref{tab:clinc_hyperparam} \\
% \bottomrule
% \end{tabular}}
% \label{tab:total_hyperparameters}
% \end{table}
We provide detailed hyperparameter settings for all baseline methods and note that we also performed hyperparameter search for these baselines. For any method-specific hyperparameters not explicitly listed, we follow the configurations recommended in their original works. The acronyms are summarized in Tab.~\ref{tab:acronyms}. Our \method{} includes three hyperparameters in total: two for the spatial \method{} ($z_{\text{thr}}$ and $\lambda_S$) and one for the inference-ready module construction ($k_{\text{pct}}$). Please refer to Tabs.~\ref{tab:cifar_hyperparam}–\ref{tab:20ng_hyperparam} for the detailed hyperparameter settings used for each task.

\subsection{Datasets and Task Details}
\label{supsec:dataset_details}

In this section, we describe the datasets used to evaluate the baselines and \method{}, and outline the task configurations for the Federated Class Incremental Learning (FCIL) setting. \method{} is evaluated on six widely used image classification benchmarks—CIFAR-100~\cite{cifar100}, ImageNet-R~\cite{imagenetr}, ImageNet-A~\cite{imageneta}, EuroSAT~\cite{helber2019eurosat}, Cars-196~\cite{krause2013cars}, and CUB-200~\cite{cub200}—as well as two text classification tasks: 20-Newsgroups~\cite{20nng} and CLINC150~\cite{clinc150}.

\subsubsection{Vision Domain}

\paragraph{CIFAR-100}

CIFAR-100~\cite{cifar100} consists of 50{,}000 training and 10{,}000 test color images of size 32$\times$32. 
The dataset spans 100 fine-grained object categories, grouped into 20 superclasses, with each class containing 600 images (500 train and 100 test). 
For the FCIL task configuration, the 100 classes are divided into 10 sequential tasks of 10 classes each.

\paragraph{ImageNet-R}

ImageNet-R~\cite{imagenetr} is a curated 30{,}000-image benchmark containing 200 ImageNet-1K categories rendered in non-photorealistic styles such as cartoons, sketches, paintings, and graphics. 
For the FCIL task configuration, the 200 classes are split into 10 tasks, each containing 20 classes.

\paragraph{ImageNet-A}

ImageNet-A~\cite{imageneta} contains approximately 7{,}500 natural images from 200 ImageNet-1K classes, filtered through an adversarial selection process that identifies samples particularly challenging for standard ImageNet-trained models. 
For the FCIL task configuration, the 200 classes are divided into 10 tasks of 20 classes each.

\paragraph{EuroSAT}

EuroSAT~\cite{helber2019eurosat} is a land-use and land-cover classification dataset constructed from Sentinel-2 satellite imagery. 
It includes 27{,}000 images with 13 spectral bands across 10 categories such as \emph{forest}, \emph{river}, \emph{residential}, and \emph{pasture}. 
For the FCIL task configuration, the 10 classes are split into 5 tasks of 2 classes each.

\paragraph{Cars-196}

Cars-196~\cite{krause2013cars} contains 16{,}185 high-resolution car images distributed across 196 fine-grained categories, where each class corresponds to a specific car make, model, and year. 
For the FCIL task configuration, we allocate 19 classes to each of the first 9 tasks and assign the remaining 25 classes to the final task.

\paragraph{CUB-200}

CUB-200~\cite{cub200} (Caltech-UCSD Birds) is a fine-grained bird classification dataset containing 11{,}788 images across 200 bird species, with rich annotations such as bounding boxes, part locations, and semantic attributes. 
For the FCIL task configuration, the 200 classes are evenly divided into 10 tasks of 20 classes each.

\subsubsection{Language Domain}

\paragraph{20-Newsgroups}

20-Newsgroups~\cite{20nng} is a widely used text classification dataset containing approximately 20{,}000 documents evenly distributed across 20 topics, including sports, politics, science, and religion. 
For the FCIL task configuration, the 20 topics are split into 10 tasks containing 2 topics each.

\paragraph{CLINC-150}

CLINC-150~\cite{clinc150} is an intent classification dataset composed of 23{,}700 user utterances spanning 150 intents across 10 domains, such as banking, credit cards, travel, and utilities. 
For the FCIL task configuration, the 150 intents are divided into 10 tasks of 15 intents each.

\begin{algorithm}[t]
\caption{\textbf{Core procedure of \methodlong{}}}
\label{alg:sum_core}
\begin{algorithmic}[1]
\footnotesize
\STATE \textbf{Input:} Task sequence $\{1\dots T\}$, number of rounds $R$, base model $\theta_{\text{pre}}$, thresholds $(z_{\text{thr}}, k_{\text{pct}})$, scaling coefficient $\lambda_S$
\STATE \textbf{Output:} Final global model $\theta_G^{(T)}$, unified vector $\tau_{\text{uni}}$, and modulators $\{\mathcal{M}_k\}$ for all $k$
\STATE Initialize $\theta_G \leftarrow \theta_{\text{pre}}$, temporal basis $\mathcal{B}\leftarrow\emptyset$

\FOR{\textbf{each task} $k = 1$ \textbf{to} $T$}

    \STATE \colorbox{NavyBlue!10}{\# Spatial \method{} across clients}

    \FOR{$t = 1$ to $R$}
        \STATE Broadcast $\theta_G$ to clients and perform local training
        \STATE Compute client vectors $\{\tau_i^{t,k}\} \leftarrow \theta_i^{t,k} - \theta_G$
        \STATE Apply z-score trimming to $\{\tau_i^{t,k}\}$
        \STATE \textbf{Surgery:} 
        $\hat{\tau}_i^{t,k} \leftarrow 
        \tau_i^{t,k} - 
        \sum_{j\neq i}
        \frac{\langle \tau_i^{t,k},\tau_j^{t,k}\rangle}
        {\|\tau_j^{t,k}\|_2^2}\tau_j^{t,k}$
        \STATE $\Delta\theta_G \leftarrow \lambda_S \sum_i \hat{\tau}_i^{t,k}$
        \STATE $\theta_G \leftarrow \theta_G + \Delta\theta_G$
    \ENDFOR

    \STATE \colorbox{green!10}{\# Temporal \method{} across tasks}

    \STATE Compute task vector $\tau_k \leftarrow \theta_G - \theta_{\text{pre}}$
    \STATE \textbf{Online Surgery:} 
    $\hat{\tau}_k \leftarrow 
    \tau_k - 
    \sum_{b\in\mathcal{B}}
    \frac{\langle \tau_k,b\rangle}{\|b\|_2^2}b$
    \STATE Append $\hat{\tau}_k$ to $\mathcal{B}$

    \STATE \colorbox{red!10}{\# Construct Inference-ready Module}

    \STATE Apply top-$k_{\text{pct}}$ sparsification to $\mathcal{B}$ to obtain $\bar{\mathcal{B}}$
    \STATE $\tau_{\text{uni}} \leftarrow \textsc{ElectSign}(\bar{\mathcal{B}})$
    \STATE Derive $\mathcal{M}_k = \{m_k, s_k\}$ and store (or deploy) $\{\tau_{\text{uni}}, \mathcal{M}_k\}$

\ENDFOR

\end{algorithmic}
\end{algorithm}

\section{Additional Experiments}
\label{sup:additional_experiments}

\subsection{Communication Cost Analysis}
\label{supsec:computation_cost_analysis}

\begin{table}[t]
\centering
\footnotesize
\caption{\textbf{Communication cost comparison across FCIL baselines.} We report upload and download sizes for each method. For \method{}, values in parentheses denote the additional download required when a user requests inference.}
\setlength{\tabcolsep}{1.5em}
\adjustbox{max width=0.55\textwidth}{
\begin{tabular}{l|cc}
\toprule
\textbf{Method} & \textbf{Download (MB)} & \textbf{Upload (MB)} \\
\midrule
ViT & 343.79 & 343.79 \\
LoRA & 6.02 & 6.02 \\
\midrule
EWC           & 343.79 & 343.79 \\
CCVR          & 343.79 & 343.86 \\
L2P           & 0.57   & 0.57 \\
CODA-P        & 15.66  & 15.66 \\
FedProto      & 343.79 & 343.83 \\
TARGET        & 343.79 & 343.79 \\
PILoRA        & 10.29  & 10.03 \\
LoRM          & 6.02   & 6.02 \\
\midrule
\rowcolor{gray!10} \method{}      & 343.79 (+187.26) & 343.79 \\
\rowcolor{gray!10} \method-LoRA   & 6.02 (+5.65)     & 6.02 \\
\bottomrule
\end{tabular}}
\label{tab:communication_cost}
\end{table}

We first analyze the communication cost incurred when a client uploads and downloads model components. In Tab.~\ref{tab:communication_cost}, we report the upload and download sizes for the main baselines. For \method{}, the numbers in parentheses denote the additional download cost when the user requests inference and needs to fetch the corresponding module. 

During this process, although \method{} involves extra download overhead due to FP16 conversion and masked-module bit packing, this is not the key bottleneck in the FCIL setting~\cite{mhanna2024countering,tziolas2025fclsurvey}.
In contrast, the actual bottleneck, the upload cost, is comparable only to that of uploading a ViT or a LoRA module. Importantly, \method{} does not require any extra upload traffic for prototypes or additional statistics. This holds without any server-side retraining, replay buffers, or synthetic data generation.

% mask -16.15 MB / 0.87 MB

\subsection{Inter-Client Vector Similarity Analysis}
\label{supsec:task_relation_analysis}

\begin{figure}[t]
    \centering
    \includegraphics[width=0.80\linewidth]{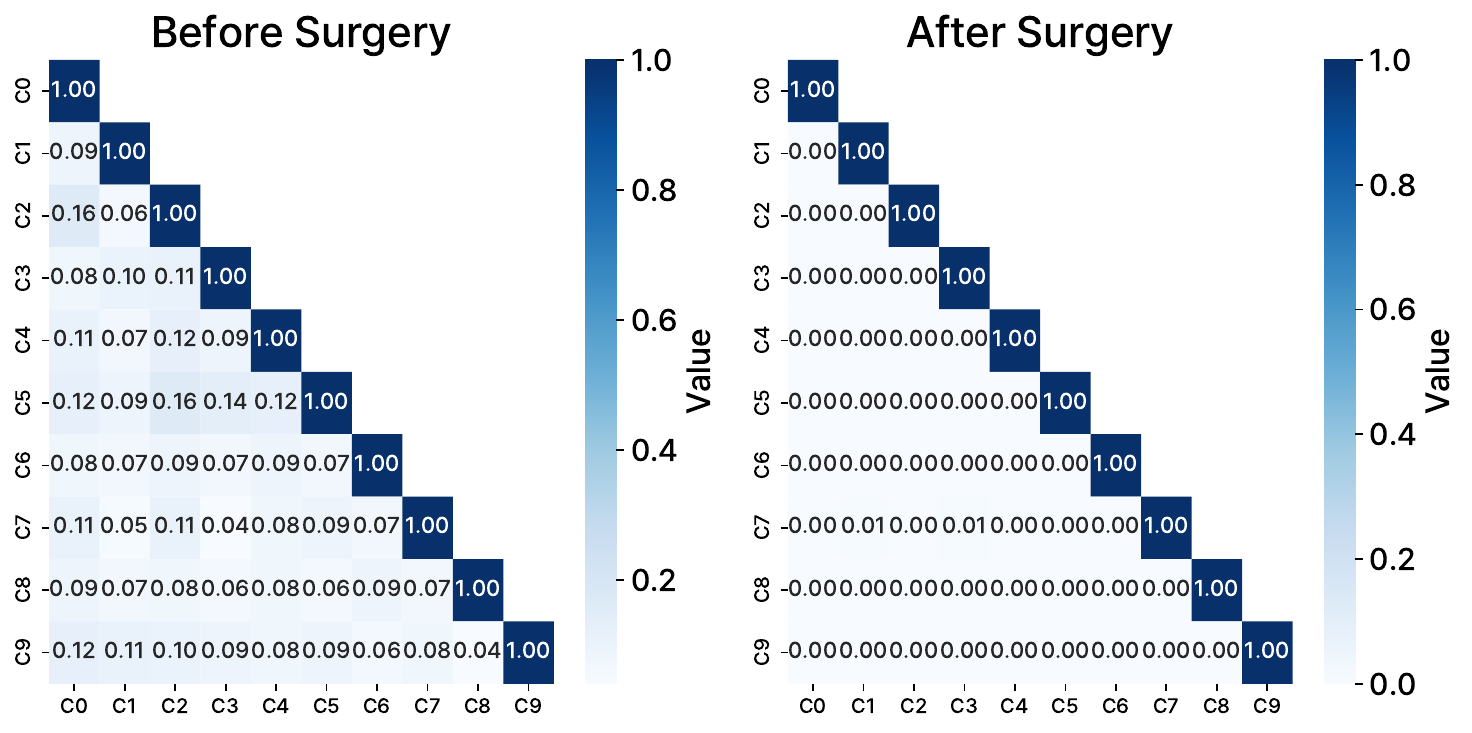}
    \caption{\textbf{Cosine similarity between inter-client vectors before and after Spatial \method{}.} Client vectors exhibit MTL-like positive alignment before surgery and become fully orthogonal afterward.}
    \label{fig:cos_inter_clients}
\end{figure}
\begin{figure}[t]
    \centering
    \begin{subfigure}{0.6\linewidth}
        \centering
        \includegraphics[width=\linewidth]{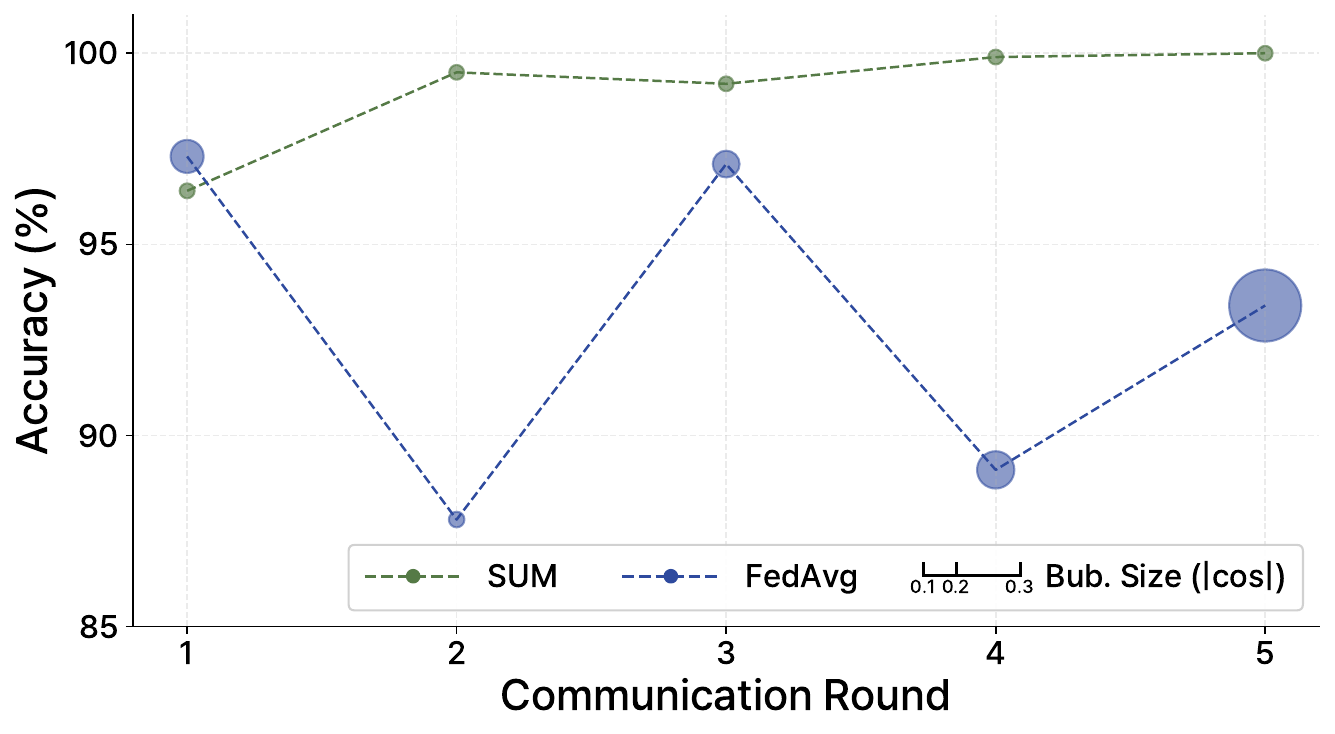}
        \caption{Evolution of average cosine similarity and accuracy across communication rounds. \method{} maintains strong performance while keeping client vectors independent.}
        \label{fig:cos_sim_round}
    \end{subfigure}
    \hfill
    \begin{subfigure}{0.35\linewidth}
        \centering
        \includegraphics[width=\linewidth]{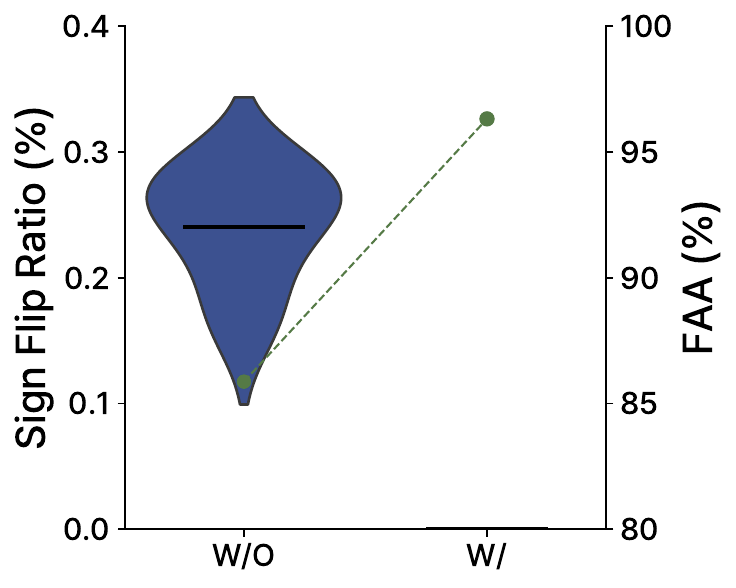}
        \caption{Sign-flip ratio during spatial \method{} with and without Z-score trimming.}
        \label{fig:z_score}
    \end{subfigure}

    \caption{Analysis across communication rounds and aggregation behavior.}
    \label{fig:cos_zscore}
\end{figure}

To support our view that FCIL can be interpreted from an MTL perspective, and that \method{} encourages effective MTL, we visualize the cosine similarity between inter-client vectors in Fig.~\ref{fig:cos_inter_clients}, both before and after the Spatial \method{} process. Even before surgery, the client vectors are almost orthogonal but remain positively aligned, resembling the structure commonly observed among task vectors in MTL~\cite{ilharco2022editing,yadav2024ties}. This observation reinforces our claim that FCIL can indeed be seen as a form of MTL.
Moreover, after surgery, all pairwise client vectors become perfect orthogonal. This aligns with findings in the model-merging literature, where orthogonality between adaptation vectors is known to be a desirable property for constructing effective merged MTL models~\cite{yang2023adamerging}.

In Fig.~\ref{fig:cos_sim_round}, we further show how the average cosine similarity between client vectors and their corresponding performance evolves as the communication rounds progress, comparing \method{} with FedAvg~\cite{fedavg}. As a result, our approach demonstrates that simply applying a \textit{Surgery-then-Merge} procedure maintains strong performance while preserving inter-client independence, effectively merging them into MTL-like client vectors.

\subsection{Effects of Z-Score Trimming}
\label{supsec:z_score_trimming}

Fig.~\ref{fig:z_score} shows the distribution of the sign-flip ratio during model aggregation in the Spatial \method{} rounds, comparing cases with and without Z-score trimming. Without trimming, the parameter sign-flip ratio is roughly 20–30\%, and prior work has shown that sign flips during the merging of adaptation vectors can significantly degrade performance~\cite{yadav2024ties,huang2024emr}.
With Z-score trimming, however, we observe that sign flips disappear entirely during vector aggregation (all values are 0.0 in the plot). Correspondingly, FAA improves by more than 10\% compared to the non-trimmed setting, highlighting the importance of this mechanism.

\subsection{Classifier-Head Merging Strategies}
\label{supsec:head_merging_choices}

 \begin{table}[t]
\centering
\footnotesize
\caption{\textbf{Comparison of classifier-head aggregation methods under different disturbance levels.}}
\setlength{\tabcolsep}{2.1em}
\adjustbox{max width=0.55\textwidth}{
\begin{tabular}{l cc}
\toprule
\multirow{2}{*}{\textbf{Method}} & \multicolumn{2}{c}{\textbf{Disturb. $\beta$}} \\
\cmidrule{2-3}
 & \textbf{0.5} & \textbf{0.01} \\
\midrule
Simple Agg.  & \textbf{97.57} & \textbf{96.36} \\
RegMean & 97.03 & 96.32 \\ 
\bottomrule
\end{tabular}}
\label{tab:merging_head}
\end{table}
 
\method{} is applied to the backbone, \ie the vision or language encoder, while for classifier-head aggregation we use RegMean as the default method~\cite{jin2022dataless}. We compare RegMean with a simple client-head aggregation baseline, where we compute each client’s classifier-head adaptation vector, sum these vectors, and add the result to the previous round’s head parameters.
As shown in Tab.~\ref{tab:merging_head}, RegMean and simple aggregation yield no noticeable performance difference, suggesting that head merging is not a major performance bottleneck in FCIL when using \method{}. That said, due to its numerical stability with respect to the scale of head adaptation vectors, we still recommend using RegMean.

\subsection{Catastrophic Forgetting Under Temporal Interference}
\label{supsec:forgetting}

\begin{figure}[t]
    \centering
    \includegraphics[width=0.65\linewidth]{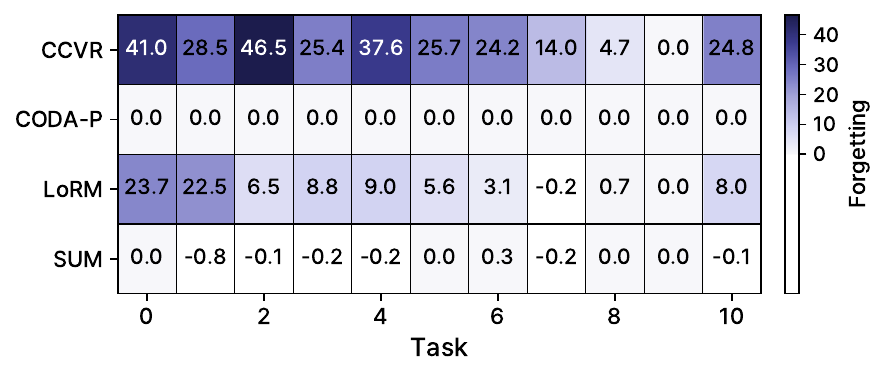}
 \caption{\textbf{Forgetting metrics in FCIL.} FL-specific methods forget heavily, CL-based methods reduce forgetting, and \method{} achieves near-zero or even negative forgetting.}
    \label{fig:forgetting}
\end{figure}

Catastrophic forgetting caused by temporal interference in FCIL is one of its key challenges and leads to severe ST-CF. In Fig.~\ref{fig:forgetting}, we report the forgetting metrics of the baseline methods. We measure forgetting as the performance difference between (i) immediately after each task and (ii) after all tasks are completed. Larger positive values indicate stronger forgetting, whereas values close to zero indicate little to no forgetting.
As a result, we observe that FL methods such as CCVR remain highly vulnerable to forgetting, whereas applying CL methods such as CODA-Prompt to FCIL helps mitigate forgetting more effectively. LoRM, an FCIL-specific method, falls somewhere in between, as it does not fully resolve the forgetting issue. In contrast, our \method{} exhibits no forgetting and even shows slight performance improvements in some cases when evaluated at the end compared to immediately after each task. This highlights the effectiveness of online surgery, sparsification, and inference modulation, which together produce more specialized task vectors and enable more reliable merging.

\subsection{Convergence Behavior and Training Loss}
\label{supsec:loss}

\begin{figure}[t]
    \centering
    \includegraphics[width=0.65\linewidth]{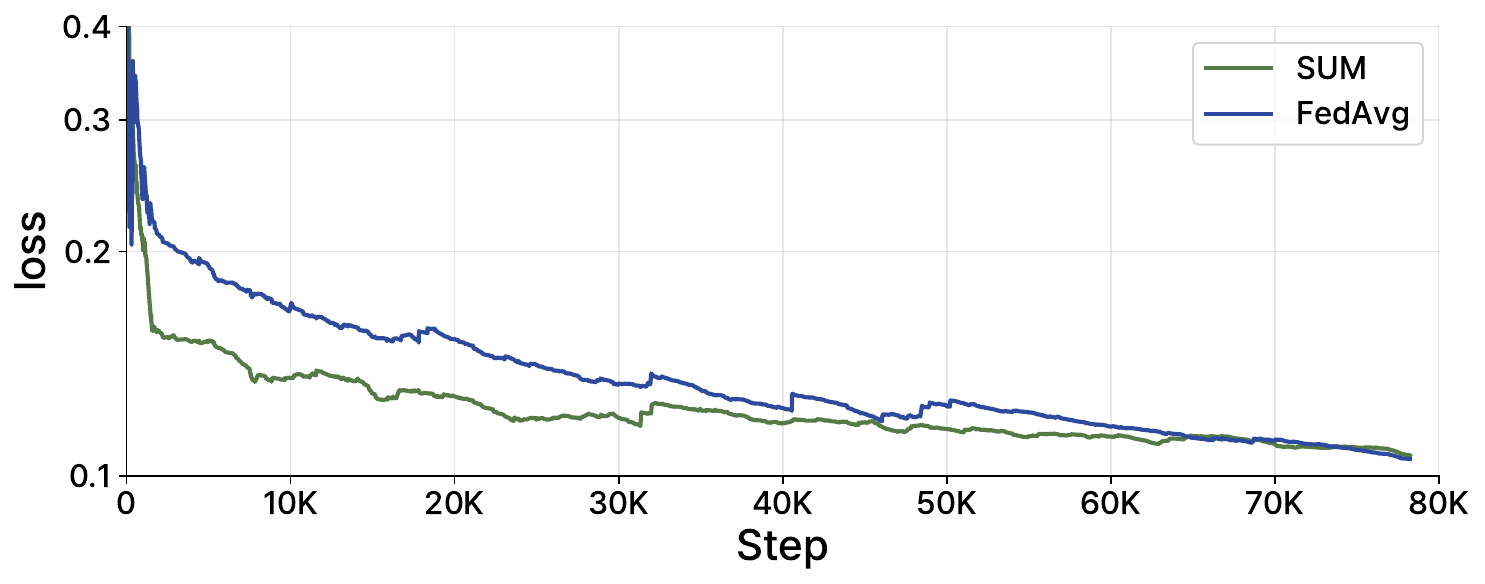}
    \caption{\textbf{Training loss comparison between \method{} and FedAvg.} \method{} exhibits lower loss in early rounds and approaches convergence more quickly than standard aggregation.}
    \label{fig:loss}
\end{figure}

Thm.~\ref{thm:main} motivates examining whether the local training-objective behavior
predicted by the aggregation-level analysis appears empirically. To validate this, Fig.\ref{fig:loss} compares the training loss of \method{} with that of FedAvg, the standard aggregation method used in FL. We observe that \method{} consistently achieves a lower training loss, especially during the early stages of learning. While both methods eventually converge to similar levels in the later rounds, \method{} reaches the convergence region more quickly due to its lower initial loss.

\begin{table}[t]
\centering
\caption{\textbf{Sharpness and ICS (Class Margin) comparison on CIFAR-100 ($\beta=0.05$).}}
\label{tab:sharpness_margin}
\setlength{\tabcolsep}{3.3pt}
\renewcommand{\arraystretch}{0.9}
\scriptsize
\begin{tabular}{lcc}
\toprule
\textbf{Method} & \textbf{Sharpness} $\downarrow$ & \textbf{Class Margin} $\uparrow$ \\
\midrule
Joint & 0.00318 & 0.362 \\
\textbf{\method{}} & \textbf{0.00077} & \textbf{0.482} \\
\bottomrule
\end{tabular}
\end{table}

\subsection{vs. Joint Training}
\label{supsec:vs_joint}

As discussed in Sec.~\ref{subsec:analysis}, we compare \method{} with centralized joint training in terms of optimization landscape and representation geometry.
Tab.~\ref{tab:sharpness_margin} reports the corresponding measurements on CIFAR-100 ($\beta=0.05$).

\newpage

\begin{table*}[t] \footnotesize
\caption{\textbf{CIFAR-100 hyperparameters.}}
\renewcommand{\arraystretch}{0.8}
\centering
\setlength{\tabcolsep}{0.75em}
\adjustbox{max width=1.0\textwidth}{
\begin{tabular}{llll} \toprule
\textbf{Distrib.} $\boldsymbol{\beta}$ & $0.5$ & $0.1$ & $0.05$ \\
\midrule
    EwC         & \textit{lr}: 1e-5 & \textit{lr}: 1e-5 & \textit{lr}: 1e-5 \\
    LwF         & \textit{lr}: 1e-5 & \textit{lr}: 1e-5 & \textit{lr}: 1e-5 \\
    FisherAVG   & \textit{lr}: 1e-5 & \textit{lr}: 1e-5 & \textit{lr}: 1e-5 \\
    RegMean     & \textit{lr}: 1e-5; $\gamma$: (0.5, 0.5) & \textit{lr}: 1e-5; $\gamma$: (0.5, 0.5) & \textit{lr}: 1e-5; $\gamma$: (0.5, 0.5) \\
    CCVR        & \textit{lr}: 1e-5 & \textit{lr}: 1e-5 & \textit{lr}: 1e-5 \\
    L2P	        & \textit{lr}: 3e-2 & \textit{lr}: 3e-2 & \textit{lr}: 3e-2 \\
    CODA-P	    & \textit{lr}: 1e-3 & \textit{lr}: 1e-3 & \textit{lr}: 1e-3 \\
    FedProto    & \textit{lr}: 1e-5 & \textit{lr}: 1e-5 & \textit{lr}: 1e-5 \\
    TARGET      & \textit{lr}: 1e-5; $\lambda_{\text{KL}}$: 25; $\textit{g}_{ep}$: 30 & 
                   \textit{lr}: 1e-5; $\lambda_{\text{KL}}$: 25; $\textit{g}_{ep}$: 30 & 
                   \textit{lr}: 1e-5; $\lambda_{\text{KL}}$: 25; $\textit{g}_{ep}$: 30 \\
    PILoRA      & \textit{lr}: 2e-2; $\textit{lr}_{pr}$: 1e-4 & \textit{lr}: 2e-2; $\textit{lr}_{pr}$: 1e-4 & \textit{lr}: 2e-2; $\textit{lr}_{pr}$: 1e-4 \\
    FOT      & \textit{lr}: 1e-5; $\epsilon_{fot}$: 0.9; $\epsilon_{inc}$: 0.0 &
                   \textit{lr}: 1e-5; $\epsilon_{fot}$: 0.9; $\epsilon_{inc}$: 0.0 &
                   \textit{lr}: 1e-5; $\epsilon_{fot}$: 0.9; $\epsilon_{inc}$: 0.0  \\    
    LoRM        & \textit{lr}: 3e-4; \textit{r}: 1; $\gamma$: (0, 0.5) &  \textit{lr}: 1e-4; \textit{r}: 16; $\gamma$: (0, 0.5) & \textit{lr}: 5e-4; \textit{r}: 16; $\gamma$: (0, 0.5) \\
    \midrule \rowcolor{gray!10}  \textbf{\method{}} & \textit{lr}: 1e-5; $z_{\text{thr}}$: 4.5 ; $\lambda_S$: 0.4 ; $k_{\text{pct}}$: 0.1 \ & \textit{lr}: 1e-5; $z_{\text{thr}}$: 4.5 ; $\lambda_S$: 0.4 ; $k_{\text{pct}}$: 0.1 & \textit{lr}: 1e-5; $z_{\text{thr}}$: 4.5 ; $\lambda_S$: 0.6 ; $k_{\text{pct}}$: 0.05 \\
\bottomrule
\end{tabular}}
\label{tab:cifar_hyperparam}
\end{table*}

\begin{table*}[t] \footnotesize
\caption{\textbf{ImageNet-R hyperparameters.}}
\renewcommand{\arraystretch}{0.8}
\centering
\setlength{\tabcolsep}{0.75em}{
\adjustbox{max width=1.0\textwidth}{
\begin{tabular}{llll} \toprule
\textbf{Distrib.} $\boldsymbol{\beta}$ & $0.5$ & $0.1$ & $0.05$ \\
\midrule
    EwC         & \textit{lr}: 1e-5 & \textit{lr}: 1e-5 & \textit{lr}: 1e-5 \\
    LwF         & \textit{lr}: 1e-5 & \textit{lr}: 1e-5 & \textit{lr}: 3e-5 \\
    FisherAVG   & \textit{lr}: 1e-5 & \textit{lr}: 1e-5 & \textit{lr}: 1e-5 \\
    RegMean     & \textit{lr}: 1e-5; $\gamma$: (0.1, 0.1) & \textit{lr}: 1e-5; $\gamma$: (0.1, 0.1) & \textit{lr}: 1e-5; $\gamma$: (0.1, 0.1) \\
    CCVR        & \textit{lr}: 1e-5 & \textit{lr}: 1e-5 & \textit{lr}: 1e-5 \\
    L2P	        & \textit{lr}: 3e-2 & \textit{lr}: 3e-2 & \textit{lr}: 3e-2 \\
    CODA-P	    & \textit{lr}: 1e-3 & \textit{lr}: 1e-3 & \textit{lr}: 1e-3 \\
    FedProto    & \textit{lr}: 1e-5 & \textit{lr}: 1e-5 & \textit{lr}: 3e-5 \\
    TARGET      & \textit{lr}: 1e-5; $\lambda_{\text{KL}}$: 25; $\textit{g}_{ep}$: 30 &
                   \textit{lr}: 1e-5; $\lambda_{\text{KL}}$: 25; $\textit{g}_{ep}$: 30 &
                   \textit{lr}: 1e-5; $\lambda_{\text{KL}}$: 25; $\textit{g}_{ep}$: 30 \\
    PILoRA      & \textit{lr}: 2e-2; $\textit{lr}_{pr}$: 1e-4 &
                   \textit{lr}: 2e-2; $\textit{lr}_{pr}$: 1e-4 &
                   \textit{lr}: 2e-2; $\textit{lr}_{pr}$: 1e-4 \\
    FOT      & \textit{lr}: 1e-4; $\epsilon_{fot}$: 0.9; $\epsilon_{inc}$: 0.0 &
                   \textit{lr}: 1e-4; $\epsilon_{fot}$: 0.9; $\epsilon_{inc}$: 0.0 &
                   \textit{lr}: 1e-4; $\epsilon_{fot}$: 0.9; $\epsilon_{inc}$: 0.0 \\
    LoRM        & \textit{lr}: 3e-3; \textit{r}: 2; $\gamma$: (0, 0.5) &  \textit{lr}: 1e-3; \textit{r}: 32; $\gamma$: (0, 0.5) & \textit{lr}: 1e-3; \textit{r}: 16; $\gamma$: (0, 0.5) \\
        \midrule \rowcolor{gray!10}  \textbf{\method{}} & \textit{lr}: 1e-5; $z_{\text{thr}}$: 4.5 ; $\lambda_S$: 1.0 ; $k_{\text{pct}}$: 0.05\ & \textit{lr}: 1e-5; $z_{\text{thr}}$: 4.5 ; $\lambda_S$: 0.6 ; $k_{\text{pct}}$: 0.05& \textit{lr}: 1e-5; $z_{\text{thr}}$: 4.5 ; $\lambda_S$: 0.8; $k_{\text{pct}}$: 0.05\\
\bottomrule
\end{tabular}}}
\label{tab:imagenetr_hyperparam}
\end{table*}

\begin{table*}[t] \footnotesize
\caption{\textbf{ImageNet-A hyperparameters.}}
\renewcommand{\arraystretch}{0.8}
\centering
\setlength{\tabcolsep}{0.75em}{
\adjustbox{max width=1.0\textwidth}{
\begin{tabular}{llll} \toprule
\textbf{Distrib.} $\boldsymbol{\beta}$ & $1.0$ & $0.5$ & $0.2$ \\
\midrule
    EwC         & \textit{lr}: 1e-5 & \textit{lr}: 1e-5 & \textit{lr}: 1e-5 \\
    LwF         & \textit{lr}: 1e-5 & \textit{lr}: 1e-5 & \textit{lr}: 3e-5 \\
    FisherAVG   & \textit{lr}: 1e-5 & \textit{lr}: 1e-5 & \textit{lr}: 1e-5 \\
    RegMean     & \textit{lr}: 1e-5; $\gamma$: (0.1, 0.1) &
                   \textit{lr}: 1e-5; $\gamma$: (0.1, 0.1) &
                   \textit{lr}: 1e-5; $\gamma$: (0.1, 0.1) \\
    CCVR        & \textit{lr}: 1e-5 & \textit{lr}: 1e-5 & \textit{lr}: 1e-5 \\
    L2P	        & \textit{lr}: 3e-2 & \textit{lr}: 3e-2 & \textit{lr}: 3e-1 \\
    CODA-P	    & \textit{lr}: 1e-2 & \textit{lr}: 1e-2 & \textit{lr}: 1e-2 \\
    FedProto    & \textit{lr}: 3e-5 & \textit{lr}: 1e-5 & \textit{lr}: 1e-5 \\
    TARGET      & \textit{lr}: 1e-4; $\lambda_{\text{KL}}$: 25; $\textit{g}_{ep}$: 30 &
                   \textit{lr}: 1e-4; $\lambda_{\text{KL}}$: 25; $\textit{g}_{ep}$: 30 &
                   \textit{lr}: 1e-4; $\lambda_{\text{KL}}$: 25; $\textit{g}_{ep}$: 30 \\
    PILoRA      & \textit{lr}: 2e-2; $\textit{lr}_{pr}$: 1e-4 &
                   \textit{lr}: 1e-2; $\textit{lr}_{pr}$: 1e-4 &
                   \textit{lr}: 2e-2; $\textit{lr}_{pr}$: 1e-4 \\
    FOT      & \textit{lr}: 1e-4; $\epsilon_{fot}$: 0.9; $\epsilon_{inc}$: 0.0 &
                   \textit{lr}: 1e-5; $\epsilon_{fot}$: 0.9; $\epsilon_{inc}$: 0.0 &
                   \textit{lr}: 1e-5; $\epsilon_{fot}$: 0.9; $\epsilon_{inc}$: 0.0  \\
    LoRM        & \textit{lr}: 1e-2; \textit{r}: 4; $\gamma$: (0, 0.5) &
                   \textit{lr}: 1e-2; \textit{r}: 4; $\gamma$: (0, 0.5) &
                   \textit{lr}: 1e-2; \textit{r}: 4; $\gamma$: (0, 0.5) \\
        \midrule \rowcolor{gray!10}  \textbf{\method{}} & \textit{lr}: 4e-5; $z_{\text{thr}}$: 4.5 ; $\lambda_S$: 0.4 ; $k_{\text{pct}}$: 0.05\ & \textit{lr}: 4e-5; $z_{\text{thr}}$: 4.5 ; $\lambda_S$: 0.3 ; $k_{\text{pct}}$: 0.05& \textit{lr}: 4e-5; $z_{\text{thr}}$: 4.5 ; $\lambda_S$: 0.3 ; $k_{\text{pct}}$: 0.05\\
\bottomrule
\end{tabular}}}
\label{tab:imageneta_hyperparam}
\end{table*}

\begin{table*}[t] \footnotesize
\caption{\textbf{EuroSAT hyperparameters.}}
\renewcommand{\arraystretch}{0.8}
\centering
\setlength{\tabcolsep}{0.75em}{
\adjustbox{max width=1.0\textwidth}{
\begin{tabular}{llll} \toprule
\textbf{Distrib.} $\boldsymbol{\beta}$ & $1.0$ & $0.5$ & $0.2$ \\
\midrule
    EwC         & \textit{lr}: 1e-5 & \textit{lr}: 1e-5 & \textit{lr}: 1e-5 \\
    LwF         & \textit{lr}: 1e-5 & \textit{lr}: 1e-5 & \textit{lr}: 3e-5 \\
    FisherAVG   & \textit{lr}: 1e-5 & \textit{lr}: 1e-5 & \textit{lr}: 1e-5 \\
    RegMean     & \textit{lr}: 1e-5; $\gamma$: (0.1, 0.1) &
                   \textit{lr}: 1e-5; $\gamma$: (0.1, 0.1) &
                   \textit{lr}: 1e-5; $\gamma$: (0.1, 0.1) \\
    CCVR        & \textit{lr}: 1e-5 & \textit{lr}: 1e-5 & \textit{lr}: 1e-5 \\
    L2P	        & \textit{lr}: 3e-2 & \textit{lr}: 3e-2 & \textit{lr}: 3e-2 \\
    CODA-P	    & \textit{lr}: 1e-3 & \textit{lr}: 1e-3 & \textit{lr}: 1e-3 \\
    FedProto    & \textit{lr}: 3e-5 & \textit{lr}: 1e-5 & \textit{lr}: 1e-5 \\
    TARGET      & \textit{lr}: 1e-5; $\lambda_{\text{KL}}$: 25; $\textit{g}_{ep}$: 30 &
                   \textit{lr}: 1e-5; $\lambda_{\text{KL}}$: 25; $\textit{g}_{ep}$: 30 &
                   \textit{lr}: 1e-5; $\lambda_{\text{KL}}$: 25; $\textit{g}_{ep}$: 30 \\
    PILoRA      & \textit{lr}: 2e-2; $\textit{lr}_{pr}$: 1e-4 &
                   \textit{lr}: 2e-2; $\textit{lr}_{pr}$: 1e-4 &
                   \textit{lr}: 2e-2; $\textit{lr}_{pr}$: 1e-4 \\
    FOT      & \textit{lr}: 1e-5; $\epsilon_{fot}$: 0.9; $\epsilon_{inc}$: 0.0 &
                   \textit{lr}: 1e-5; $\epsilon_{fot}$: 0.9; $\epsilon_{inc}$: 0.0 &
                   \textit{lr}: 1e-5; $\epsilon_{fot}$: 0.9; $\epsilon_{inc}$: 0.0  \\
    LoRM        & \textit{lr}: 3e-3; \textit{r}: 1; $\gamma$: (0, 0.5) &
                   \textit{lr}: 3e-3; \textit{r}: 1; $\gamma$: (0, 0.5) &
                   \textit{lr}: 1e-3; \textit{r}: 4; $\gamma$: (0, 0.5) \\
        \midrule \rowcolor{gray!10}  \textbf{\method{}} & \textit{lr}: 3e-5; $z_{\text{thr}}$: 4.5 ; $\lambda_S$: 0.2 ; $k_{\text{pct}}$: 0.1\ & \textit{lr}: 3e-5; $z_{\text{thr}}$: 4.5 ; $\lambda_S$: 0.1 ; $k_{\text{pct}}$: 0.1 & \textit{lr}: 3e-5; $z_{\text{thr}}$: 4.5 ; $\lambda_S$: 0.3 ; $k_{\text{pct}}$: 0.1\\
\bottomrule
\end{tabular}}}
\label{tab:eurosat_hyperparam}
\end{table*}

\begin{table*}[t] \footnotesize
\caption{\textbf{Cars-196 hyperparameters.}}
\renewcommand{\arraystretch}{0.8}
\centering
\setlength{\tabcolsep}{0.75em}{
\adjustbox{max width=1.0\textwidth}{
\begin{tabular}{llll} \toprule
\textbf{Distrib.} $\boldsymbol{\beta}$ & $1.0$ & $0.5$ & $0.2$ \\
\midrule
    EwC         & \textit{lr}: 1e-5 & \textit{lr}: 1e-5 & \textit{lr}: 1e-5 \\
    LwF         & \textit{lr}: 1e-5 & \textit{lr}: 1e-5 & \textit{lr}: 3e-5 \\
    FisherAVG   & \textit{lr}: 1e-5 & \textit{lr}: 1e-5 & \textit{lr}: 1e-5 \\
    RegMean     & \textit{lr}: 1e-5; $\gamma$: (0.1, 0.1) &
                   \textit{lr}: 1e-5; $\gamma$: (0.1, 0.1) &
                   \textit{lr}: 1e-5; $\gamma$: (0.1, 0.1) \\
    CCVR        & \textit{lr}: 1e-5 & \textit{lr}: 1e-5 & \textit{lr}: 1e-5 \\
    L2P	        & \textit{lr}: 3e-2 & \textit{lr}: 3e-2 & \textit{lr}: 3e-2 \\
    CODA-P	    & \textit{lr}: 3e-2 & \textit{lr}: 3e-2 & \textit{lr}: 3e-2 \\
    FedProto    & \textit{lr}: 1e-5 & \textit{lr}: 1e-5 & \textit{lr}: 1e-5 \\
    TARGET      & \textit{lr}: 1e-4; $\lambda_{\text{KL}}$: 25; $\textit{g}_{ep}$: 30 &
                   \textit{lr}: 1e-4; $\lambda_{\text{KL}}$: 25; $\textit{g}_{ep}$: 30 &
                   \textit{lr}: 1e-4; $\lambda_{\text{KL}}$: 25; $\textit{g}_{ep}$: 30 \\
    PILoRA      & \textit{lr}: 1e-1; $\textit{lr}_{pr}$: 1e-4 &
                   \textit{lr}: 1e-1; $\textit{lr}_{pr}$: 1e-4 &
                   \textit{lr}: 1e-1; $\textit{lr}_{pr}$: 1e-4 \\
    FOT     & \textit{lr}: 1e-4; $\epsilon_{fot}$: 0.9; $\epsilon_{inc}$: 0.0 &
                   \textit{lr}: 1e-4; $\epsilon_{fot}$: 0.9; $\epsilon_{inc}$: 0.0 &
                   \textit{lr}: 1e-4; $\epsilon_{fot}$: 0.9; $\epsilon_{inc}$: 0.0  \\
    LoRM        & \textit{lr}: 1e-2; \textit{r}: 8; $\gamma$: (0, 0.5) &
                   \textit{lr}: 1e-2; \textit{r}: 8; $\gamma$: (0, 0.5) &
                   \textit{lr}: 1e-2; \textit{r}: 4; $\gamma$: (0, 0.5) \\
        \midrule \rowcolor{gray!10}  \textbf{\method{}} & \textit{lr}: 5e-5; $z_{\text{thr}}$: 4.5 ; $\lambda_S$: 0.2 ; $k_{\text{pct}}$: 0.1\ & \textit{lr}: 5e-5; $z_{\text{thr}}$: 4.5 ; $\lambda_S$: 0.3 ; $k_{\text{pct}}$: 0.1 & \textit{lr}: 5e-5; $z_{\text{thr}}$: 4.5 ; $\lambda_S$: 0.2 ; $k_{\text{pct}}$: 0.1\\
\bottomrule
\end{tabular}}}
\label{tab:cars_hyperparam}
\end{table*}

\begin{table*}[t] \footnotesize
\caption{\textbf{CUB-200 hyperparameters.}}
\renewcommand{\arraystretch}{0.8}
\centering
\setlength{\tabcolsep}{0.75em}{
\adjustbox{max width=1.0\textwidth}{
\begin{tabular}{llll} \toprule
\textbf{Distrib.} $\boldsymbol{\beta}$ & $1.0$ & $0.5$ & $0.2$ \\
\midrule
    EwC         & \textit{lr}: 1e-5 & \textit{lr}: 1e-5 & \textit{lr}: 1e-5 \\
    LwF         & \textit{lr}: 1e-5 & \textit{lr}: 1e-5 & \textit{lr}: 3e-5 \\
    FisherAVG   & \textit{lr}: 1e-5 & \textit{lr}: 1e-5 & \textit{lr}: 1e-5 \\
    RegMean     & \textit{lr}: 1e-5; $\gamma$: (0.1, 0.1) &
                   \textit{lr}: 1e-5; $\gamma$: (0.1, 0.1) &
                   \textit{lr}: 1e-5; $\gamma$: (0.1, 0.1) \\
    CCVR        & \textit{lr}: 1e-5 & \textit{lr}: 1e-5 & \textit{lr}: 1e-5 \\
    L2P	        & \textit{lr}: 3e-1 & \textit{lr}: 3e-1 & \textit{lr}: 3e-1 \\
    CODA-P	    & \textit{lr}: 1e-3 & \textit{lr}: 1e-3 & \textit{lr}: 1e-3 \\
    FedProto    & \textit{lr}: 1e-5 & \textit{lr}: 1e-5 & \textit{lr}: 1e-5 \\
    TARGET      & \textit{lr}: 1e-4; $\lambda_{\text{KL}}$: 25; $\textit{g}_{ep}$: 30 &
                   \textit{lr}: 1e-4; $\lambda_{\text{KL}}$: 25; $\textit{g}_{ep}$: 30 &
                   \textit{lr}: 1e-4; $\lambda_{\text{KL}}$: 25; $\textit{g}_{ep}$: 30 \\
    PILoRA      & \textit{lr}: 1; $\textit{lr}_{pr}$: 1e-4 &
                   \textit{lr}: 1; $\textit{lr}_{pr}$: 1e-4 &
                   \textit{lr}: 1; $\textit{lr}_{pr}$: 1e-4 \\
FOT      & \textit{lr}: 1e-4; $\epsilon_{fot}$: 0.9; $\epsilon_{inc}$: 0.0 &
                   \textit{lr}: 1e-4; $\epsilon_{fot}$: 0.9; $\epsilon_{inc}$: 0.0 &
                   \textit{lr}: 1e-4; $\epsilon_{fot}$: 0.9; $\epsilon_{inc}$: 0.0 \\
    LoRM        & \textit{lr}: 1e-2; \textit{r}: 1; $\gamma$: (0, 0.3) &
                   \textit{lr}: 3e-2; \textit{r}: 1; $\gamma$: (0, 0.3) &
                   \textit{lr}: 3e-2; \textit{r}: 1; $\gamma$: (0, 0.3) \\
        \midrule \rowcolor{gray!10}  \textbf{\method{}} & \textit{lr}: 1e-4; $z_{\text{thr}}$: 4.5 ; $\lambda_S$: 0.2 ; $k_{\text{pct}}$: 0.1\ & \textit{lr}: 1e-4; $z_{\text{thr}}$: 4.5 ; $\lambda_S$: 0.1 ; $k_{\text{pct}}$: 0.05& \textit{lr}: 1e-4 ; $z_{\text{thr}}$: 4.5 ; $\lambda_S$: 0.1 ; $k_{\text{pct}}$: 0.1\\
\bottomrule
\end{tabular}}}
\label{tab:cub_hyperparam}
\end{table*}

\begin{table*}[t] \footnotesize
\caption{\textbf{20-NewsGroups hyperparameters.}}
\renewcommand{\arraystretch}{0.8}
\centering
\setlength{\tabcolsep}{0.75em}{
\adjustbox{max width=1.0\textwidth}{
\begin{tabular}{llll} \toprule
\textbf{Distrib.} $\boldsymbol{\beta}$ & $1.0$ & $0.5$ & $0.2$ \\
\midrule
    RegMean     & \textit{lr}: 3e-3; $\gamma$: (0.5, 0.5) &
                   \textit{lr}: 3e-3; $\gamma$: (0.5, 0.5) &
                   \textit{lr}: 3e-3; $\gamma$: (0.5, 0.5) \\
    CCVR        & \textit{lr}: 3e-4 & \textit{lr}: 3e-4 & \textit{lr}: 1e-3 \\
    CODA-P & \textit{lr}: 1e-3 & \textit{lr}: 1e-3 & \textit{lr}: 1e-3 \\
    LoRM        & \textit{lr}: 1e-2; \textit{r}: 2; $\gamma$: (0, 0.5) &
                   \textit{lr}: 1e-2; \textit{r}: 1; $\gamma$: (0, 0.5) &
                   \textit{lr}: 1e-2; \textit{r}: 1; $\gamma$: (0, 0.1) \\
        \midrule \rowcolor{gray!10}  \textbf{\method{}} & \textit{lr}: 3e-3; $z_{\text{thr}}$: 4.5 ; $\lambda_S$: 0.1 ; $k_{\text{pct}}$: 0.05\ & \textit{lr}: 3e-3; $z_{\text{thr}}$: 4.5 ; $\lambda_S$: 0.2 ; $k_{\text{pct}}$: 0.05& \textit{lr}: 3e-3; $z_{\text{thr}}$: 4.5 ; $\lambda_S$: 0.2 ; $k_{\text{pct}}$: 0.05\\
\bottomrule
\end{tabular}}}
\label{tab:20ng_hyperparam}
\end{table*}

\begin{table*}[t] \footnotesize
\caption{\textbf{CLINC-150 hyperparameters.}}
\renewcommand{\arraystretch}{0.8}
\centering
\setlength{\tabcolsep}{0.75em}{
\adjustbox{max width=1.0\textwidth}{
\begin{tabular}{llll} \toprule
\textbf{Distrib.} $\boldsymbol{\beta}$ & $1.0$ & $0.5$ & $0.2$ \\
\midrule
    RegMean     & \textit{lr}: 3e-3; $\gamma$: (0.5, 0.5) &
                   \textit{lr}: 3e-3; $\gamma$: (0.5, 0.5) &
                   \textit{lr}: 3e-3; $\gamma$: (0.5, 0.5) \\
    CCVR        & \textit{lr}: 3e-4 & \textit{lr}: 3e-4 & \textit{lr}: 1e-3 \\
    CODA-P & \textit{lr}: 1e-3 & \textit{lr}: 1e-3 & \textit{lr}: 1e-3 \\
    LoRM        & \textit{lr}: 1e-2; \textit{r}: 2; $\gamma$: (0, 0.5) &
                   \textit{lr}: 1e-2; \textit{r}: 1; $\gamma$: (0, 0.5) &
                   \textit{lr}: 1e-2; \textit{r}: 1; $\gamma$: (0, 0.1) \\
        \midrule \rowcolor{gray!10}  \textbf{\method{}} & \textit{lr}: 3e-3; $z_{\text{thr}}$: 4.5 ; $\lambda_S$: 0.2 ; $k_{\text{pct}}$: 0.05\ & \textit{lr}: 5e-3; $z_{\text{thr}}$: 4.5 ; $\lambda_S$: 0.1 ; $k_{\text{pct}}$: 0.05 & \textit{lr}: 2e-3; $z_{\text{thr}}$: 4.5 ; $\lambda_S$: 0.2 ; $k_{\text{pct}}$: 0.05\\
\bottomrule
\end{tabular}}}
\label{tab:clinc_hyperparam}
\end{table*}

\end{document}